\documentclass[sigconf]{acmart}
\AtBeginDocument{%
  }

\setcopyright{acmlicensed}

\copyrightyear{2026}
\acmYear{2026}
\setcopyright{cc}
\setcctype{by}
\acmConference[WWW '26]{Proceedings of the ACM Web Conference 2026}{April 13--17, 2026}{Dubai, United Arab Emirates}
\acmBooktitle{Proceedings of the ACM Web Conference 2026 (WWW '26), April 13--17, 2026, Dubai, United Arab Emirates}
\acmPrice{}
\acmDOI{10.1145/3774904.3792626}
\acmISBN{979-8-4007-2307-0/2026/04}

\usepackage{amsmath}
\usepackage{booktabs}        % \toprule, \midrule, \bottomrule
\newcommand{\best}[1]{\cellcolor{gray!30}\textbf{#1}}
\newcommand{\second}[1]{\cellcolor{gray!18}\textbf{#1}}
\newcommand{\third}[1]{\cellcolor{gray!10}\textbf{#1}}
\usepackage[table]{xcolor}   % 颜色，支持 \colorbox{}、灰度 gray!<num>
\usepackage{colortbl}        % \cellcolor 用在表格单元格

\usepackage{subcaption}      % 提供 \subcaption*（内部会自动加载 caption 包）

\usepackage{graphicx}        % \resizebox、\includegraphics

\usepackage{hyperref}        % \autoref、链接

\usepackage{amsthm}

\newtheorem{remark}{Remark}

\usepackage{tikz}
\usetikzlibrary{arrows.meta,positioning,fit,shapes,shapes.multipart,calc,backgrounds}

\usepackage[noend]{algpseudocode} % 无结束标记的算法包
\usepackage{algorithm}            % 算法浮动环境
\begin{document}

\title{Prototype-Aligned Federated Soft-Prompts for Continual Web Personalization}


\author{Canran Xiao}
\email{xiaocr3@mail.sysu.edu.cn} % TODO: 请填写通讯作者邮箱
\orcid{https://orcid.org/0009-0001-9730-5454} %
\affiliation{%
  \institution{Shenzhen Campus of Sun Yat-sen University}
  \city{Shenzhen}
  \state{Guangdong}
  \country{China}
}

\author{Liwei Hou}
\authornote{Corresponding author.}
\email{houliwei@hnu.edu.cn} % 
\orcid{https://orcid.org/0000-0003-3476-2452} % TODO: 请填写Zhiming Lin的ORCID ID
\affiliation{%
  \institution{Hunan University}
  \city{Changsha} % TODO: 请填写城市
  \state{Hunan}
  \country{China}
}

\renewcommand{\shortauthors}{Canran Xiao and Liwei Hou}
%% No italics, no superscripts, not anonymous
%% Use footnote or author note to identify equal contribution and/or contact author info

\begin{abstract}
Continual Web personalization is essential for engagement, yet real-world non-stationarity and privacy constraints make it hard to adapt quickly without forgetting long-term preferences. We target this gap by seeking a privacy-conscious, parameter-efficient interface that controls stability--plasticity at the user/session level while tying user memory to a shared semantic prior. We propose ProtoFed-SP, a prompt-based framework that injects dual-timescale soft prompts into a frozen backbone: a fast, sparse short-term prompt tracks session intent, while a slow long-term prompt is anchored to a small server-side prototype library that is continually refreshed via differentially private federated aggregation. Queries are routed to Top-$M$ prototypes to compose a personalized prompt. Across eight benchmarks, ProtoFed-SP improves NDCG@10 by +2.9\% and HR@10 by +2.0\% over the strongest baselines, with notable gains on Amazon-Books (+5.0\% NDCG vs.\ INFER), H\&M (+2.5\% vs.\ Dual-LoRA), and Taobao (+2.2\% vs.\ FedRAP). It also lowers forgetting (AF) and Steps-to-95\% and preserves accuracy under practical DP budgets. Our contribution is a unifying, privacy-aware prompting interface with prototype anchoring that delivers robust continual personalization and offers a transparent, controllable mechanism to balance stability and plasticity in deployment.
\end{abstract}




\iffalse

\begin{teaserfigure}
  \includegraphics[width=\textwidth]{sampleteaser}
  \caption{Seattle Mariners at Spring Training, 2010.}
  \Description{Enjoying the baseball game from the third-base
  seats. Ichiro Suzuki preparing to bat.}
  \label{fig:teaser}
\end{teaserfigure}

\fi


\maketitle
\section{Introduction}
Personalized Web experiences---from recommendations\cite{liang2006personalized,rana2024user,li2025multi} to on-site search and feeds\cite{cai2025large,zhou2024cognitive}---operate in non-stationary environments where user interests, item supply, and platform rules evolve continually. This non-stationarity makes accurate, timely adaptation indispensable for engagement and safety at scale, yet exposes a classic stability--plasticity dilemma: models must absorb fresh signals quickly without erasing long-term preferences. At the same time, modern deployments increasingly require on-device adaptation and privacy-preserving training, challenging conventional full retraining pipelines\cite{wu2021effective,zeng2025lensllmunveilingfinetuningdynamics,lin2025planbudgeteffectiveefficient,yao2023ndc}.

Recent advancements in continuous recommendation cover self-correction, regularization, federated personalization, and parameter-efficient prompting—showing that models are capable of adapting under non-stationary conditions~\citep{Zhang2024INFER,Wang2023SAILPIW,lee2024continual,Li2020FedProx,Li2021Ditto,xiao2024confusion,Li2021PrefixTuning,Lester2021PromptTuning,Hu2022LoRA,Bao2025RecICL,Yang2025PCL}. 
Yet across these lines, three high-level limitations remain. 
First, adaptation is typically governed at the global or loss level, with no portable, shareable semantic anchor that tethers per-user memory to population structure—leaving stability–plasticity largely uncontrolled at the user granularity~\citep{li2023ultrare,tao2023dudb}. 
Second, practical deployments must respect privacy and bandwidth: replay and interaction-derived signals can be costly or sensitive, while federated methods often personalize via extra weights without a lightweight, prompt-level memory and explicit DP accounting~\citep{McMahan2017FedAvg,Li2020FedProx,Li2021Ditto,wang2024computing}. 
Third, although PEFT/LLM approaches enable low-overhead adaptation, they rarely couple fast session plasticity with a refreshable global prior that is updated in a privacy-aware, federated manner~\citep{Li2021PrefixTuning,Lester2021PromptTuning,Hu2022LoRA,Bao2025RecICL,Yang2025PCL,zhang2024cf,zhang2025enhancing}. 
In short, the field lacks a \textbf{privacy-conscious, parameter-efficient} personalization interface that (i) provides per-user control of stability–plasticity and (ii) binds user memory to a shared semantic prior that can be continually refreshed under federated constraints (with fairness monitored over time~\citep{Yoo2024FADE,li2024towards,wang2025medical}).

This paper addresses that gap by studying a \emph{prompt-based continual personalization interface} that unifies per-user, dual-timescale adaptation with a small population semantic prior maintained in federated form under differential privacy. We give the answer to the following question: Can we achieve fast session-level plasticity and long-term stability per user, while privately refreshing a global anchor that resists drift?

Our contributions can be summarized as follows:

\textbf{(1) Insight.} We articulate and operationalize a stability--plasticity control interface for continual personalization: fast session updates are decoupled from slow preference memory, and slow memory is anchored to a compact, federated semantic prior---yielding a concrete mechanism to reason about drift, forgetting, and adaptation across users. 

\textbf{(2) Performance.} Across eight benchmarks, the proposed interface attains consistent gains over strong SoTA, while preserving utility under practical DP budgets. 
 
\textbf{(3) Capability.} The design is \emph{parameter- and communication-efficient}, supports federated and multi-modal settings, and improves cold-start and high-drift segments—demonstrating that privacy-conscious, anchored prompting is a viable path toward robust continual Web personalization.

\section{Related Work}
\subsection{Continual adaptation in recommender systems}
Early industrial CL for RS emphasized self-correction and error re-use. ReLoop~\citep{Cai2022ReLoop} encourages each model version to reduce the predecessor’s errors; ReLoop2~\citep{Zhu2023ReLoop2} extends this line by coupling slow (parametric) and fast (non-parametric error memory) learners for responsive adaptation. Replay has evolved from reservoir or frequency heuristics to influence-based selection (INFER) that prioritizes exemplars with higher training impact~\citep{Zhang2024INFER}. On the regularization side, SAIL-PIW learns user-wise imitation weights to balance recency vs.\ history, hinting at per-user stability–plasticity~\citep{Wang2023SAILPIW}; CCD co-evolves teacher–student via continual distillation~\citep{lee2024continual}. Despite their progress, these methods typically (a) treat stability–plasticity at the algorithm or global level rather than at the prompt/memory level for each user; (b) rely on storing interactions or logits (privacy/cost); and (c) lack an explicit, shareable anchor that tethers long-term user memory to population semantics. Our approach differs by (1) representing user memory as \emph{dual-timescale soft prompts}, (2) aligning the slow (long-term) prompt to a federated prototype library to curb drift, and (3) avoiding raw-data replay through DP aggregation of compressed prompt embeddings.

\subsection{Federated and privacy-preserving personalization}
Foundational FL work like FedProx established communication-efficient aggregation and robustness to statistical heterogeneity~\citep{McMahan2017FedAvg,Li2020FedProx}. Personalized FL (PFL) then introduced per-client objectives (e.g., Ditto) to reconcile accuracy, robustness and fairness~\citep{Li2021Ditto,chen2024integration}. Taobao’s CTNet formalized continual transfer between time-evolving domains, showing the value of preserving historical parameters under dynamics~\citep{Liu2023CTNet}. However, most PFL/CTL methods do not expose a prompt-level memory for fast personalization, rarely provide DP guarantees on uploaded signals, and do not unify per-user stability–plasticity control with a server-side semantic prior~\citep{jiang2025transforming}. We bridge these gaps by (i) learning user prompts (fast \& slow) on device, (ii) updating population prototypes via DP-aware aggregation, and (iii) routing queries to a small set of prototypes for parameter-efficient, privacy-conscious adaptation.

\subsection{Parameter-efficient prompting and LLM-based recommendation}
PEFT replaces full fine-tuning with small trainable modules: prefix/prompt tuning learn continuous prompts for frozen LMs~\citep{Li2021PrefixTuning,Lester2021PromptTuning,chen2024post}; LoRA injects low-rank adapters~\citep{Hu2022LoRA}. Recent LLM4Rec studies tailor PEFT/ICL to dynamics: RecICL adapts to recency purely via in-context exemplars without weight updates~\citep{Bao2025RecICL}; PCL treats prompts as external memory to mitigate forgetting in continual user modeling~\citep{Yang2025PCL}. While compelling, these approaches either lack a shareable semantic prior (each user/task learns its own prompt/adapters), or trade off fast session-level adaptation against long-term stability without an explicit anchor or federated privacy~\citep{yao2024swift}. Our method complements PEFT/LLM4Rec by tying user prompts to a prototype manifold: a small, DP-refreshed library that stabilizes long-term memory, while short-term prompts provide rapid session-level plasticity. This yields a concrete mechanism for per-user stability–plasticity control that coexists with FL constraints.

\section{Method: Prototype-Aligned Federated Soft-Prompts (ProtoFed-SP)}
\label{sec:method}

ProtoFed-SP treats soft prompts as user-specific, parameter-efficient memories and prototypes as server-side, federated anchors of stable population knowledge. Each client $u$ maintains a dual-timescale prompt state $(p^{\text{long}}_u,\,p^{\text{short}}_{u,t})$ for slow/fast preference dynamics and queries a global prototype library $\mathcal{C}=\{c_k\}_{k=1}^K$ to compose a routed prompt $\tilde p_{u,t}$ for inference. The long-term prompt is explicitly aligned to prototypes via an alignment loss, while the prototype library is federatedly updated by differentially private (DP) aggregation of compressed client-side prompt embeddings. The overall objective balances recommendation accuracy, alignment (stability), and sparsity (plasticity), as shown in Fig.~\ref{fig:framework}.

\begin{figure*}[htb]
	\centering
	\includegraphics[width=\linewidth]{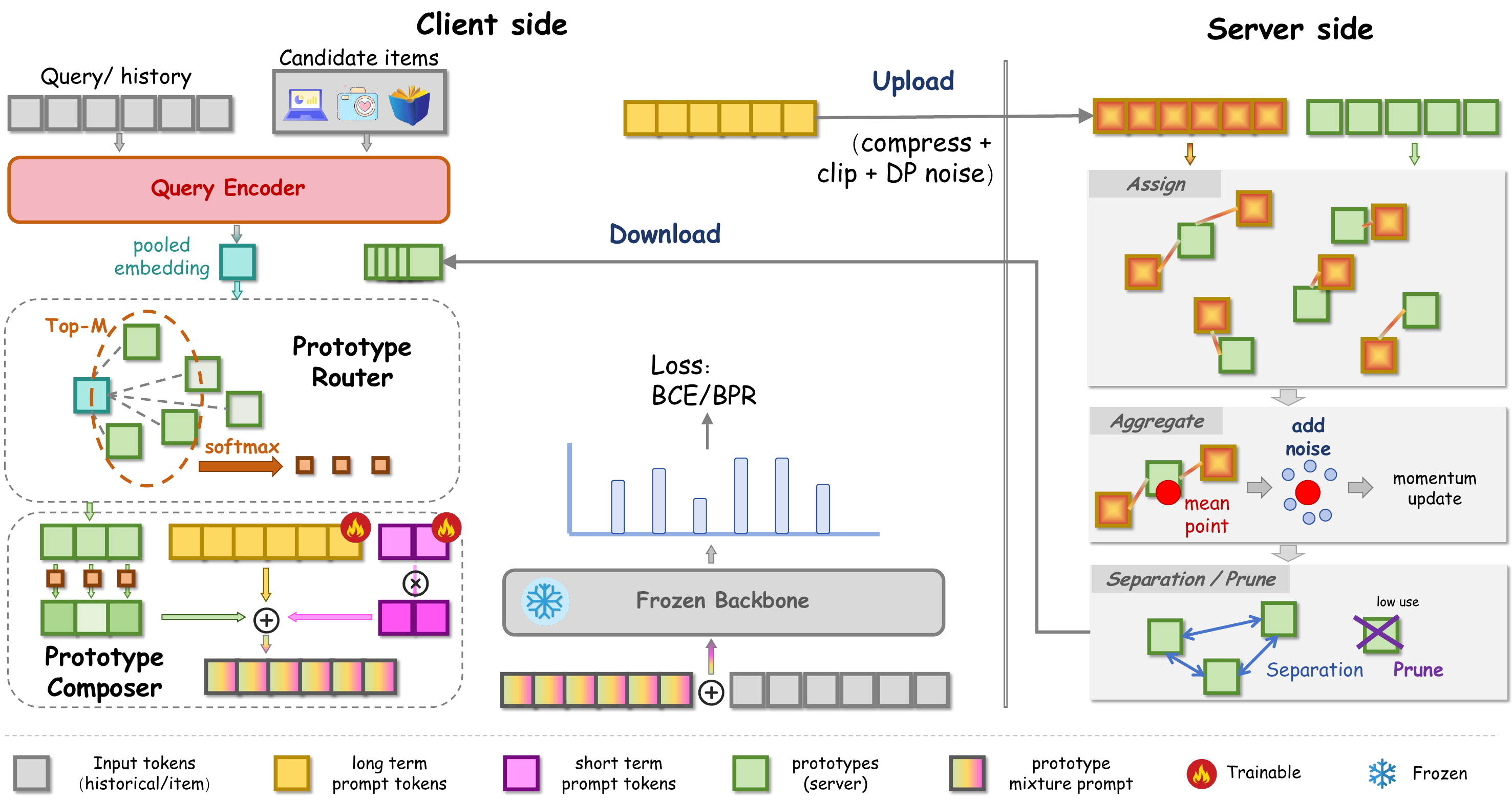}
	\caption{
\textbf{ProtoFed-SP: Prototype-Aligned Federated Soft-Prompt Personalization.}
\emph{Left (Client side).}
Given a user query or interaction history, the client encodes the context into a pooled embedding and routes it to a small set of Top-$M$ population prototypes via similarity-based retrieval.
A dual-timescale prompt state is maintained locally: a \textbf{long-term prompt} capturing stable preferences and a \textbf{short-term prompt} modeling session-level intent.
The routed prompt is composed by summing the long-term prompt, a drift-adaptive short-term prompt, and a weighted mixture of retrieved prototypes, and is injected into a frozen backbone for ranking.
Gradients from the recommendation loss (BCE/BPR) update only the prompt parameters, with sparse proximal updates for the short-term prompt and prototype-aligned updates for the long-term prompt.
\emph{Right (Server side).}
Clients periodically upload compressed long-term prompt embeddings protected by clipping and Gaussian noise.
The server performs differentially private aggregation (e.g., DP-FedKMeans or robust alternatives) to update a compact prototype library, with momentum updates, prototype separation, and pruning to prevent collapse.
The refreshed prototypes are broadcast back to clients, enabling privacy-conscious, parameter-efficient continual personalization with explicit stability--plasticity control.
}
	\label{fig:framework}
\end{figure*}

\subsection{Notation and Setting}
Let $\mathcal{U}$ be the user set. At time $t\in\mathbb{N}$, user $u\in\mathcal{U}$ issues a request with query/session context $q_{u,t}$ and local interaction buffer $\mathcal{D}_{u,t}$. A frozen backbone recommender $f_\theta$ (sequence/ranking model with parameters $\theta$) consumes item features $x_{u,t}$, while a \emph{soft prompt} $\tilde p_{u,t}\in\mathbb{R}^{L_p\times d}$ is injected at the embedding layer. Each client keeps a \emph{long-term} prompt $p^{\text{long}}_u\in\mathbb{R}^{L_p\times d}$ and a \emph{short-term} prompt $p^{\text{short}}_{u,t}\in\mathbb{R}^{L_p\times d}$. The server maintains a prototype library $\mathcal{C}=\{c_k\}_{k=1}^K$, $c_k\in\mathbb{R}^{L_p\times d}$ (or an equivalent latent representation). We denote by $\phi:\mathbb{R}^{L_p\times d}\to\mathbb{R}^{d_\phi}$ a prompt encoder used for alignment/aggregation (e.g., average pooling followed by an MLP), and by $g:\mathcal{Q}\to\mathbb{R}^{d_\phi}$ a query encoder for routing.

\subsection{Client: Dual-Timescale Personalized Prompting}
Users exhibit both slowly evolving preferences (genres, brands) and rapidly changing intents (session-level topics). A single prompt state either overfits to recent noise (plasticity dominates) or forgets old habits (stability dominates). We separate the temporal roles.

The composed prompt used for inference is
\begin{equation}
\tilde p_{u,t}
\;=\;
p^{\text{long}}_u
\;+\;
\alpha_{u,t}\,p^{\text{short}}_{u,t}
\;+\;
\sum_{k\in \text{Top-}M(q_{u,t})}
w_k(q_{u,t})\,c_k,
\label{eq:compose}
\end{equation}
where $\alpha_{u,t}\in\mathbb{R}_{\ge 0}$ scales short-term influence, $w_k(q_{u,t})\ge 0$ are routing weights with $\sum_k w_k=1$, and $c_k$ are prototypes. The \emph{prompt injection} concatenates $\tilde p_{u,t}$ with token embeddings: if $E(\cdot)$ maps tokens to $\mathbb{R}^d$, we form $\mathrm{Inject}(E(x_{u,t}),\tilde p_{u,t})=[\tilde p_{u,t};\,E(x_{u,t})]$ and feed into $f_\theta$.

\paragraph{Short-term update (fast, sparse).}
To avoid prompt sprawl, we promote sparsity in $p^{\text{short}}_{u,t}$ via an $\ell_1$ penalty and update it with a proximal step:
\begin{align}
\!\!\!\!p^{\text{short}}_{u,t}
\leftarrow
\mathrm{SoftThresh}
\Big(
p^{\text{short}}_{u,t}
-\eta_s\,\nabla_{p^{\text{short}}_{u,t}}\mathcal{L}_{\text{rec}}(u,t)
,\;\eta_s\lambda_p
\Big),
\label{eq:short_update}
\end{align}
where $\eta_s>0$ is the step size, $\lambda_p>0$ controls sparsity, and $\mathrm{SoftThresh}(z,\tau)=\mathrm{sign}(z)\cdot\max\{|z|-\tau,0\}$ is applied element-wise. Here $\mathcal{L}_{\text{rec}}$ is the recommendation loss (defined in \S\ref{subsec:loss}).

\paragraph{Long-term update (slow, aligned).}
Long-term prompts incorporate an alignment proximal to prototypes (\S\ref{subsec:align}):
\begin{align}
p^{\text{long}}_u \leftarrow \arg\min_{p} 
\bigg\{ 
& \underbrace{\langle \nabla_{p}\mathcal{L}_{\text{rec}}(u,t),\,p-p^{\text{long}}_u\rangle
+\tfrac{1}{2\eta_\ell}\|p-p^{\text{long}}_u\|_F^2}_{\text{quadratic model of }\mathcal{L}_{\text{rec}}} \nonumber \\
& + \lambda_s\,\underbrace{\mathcal{A}\!\left(\phi(p),\,\mathcal{C}\right)}_{\text{prototype alignment}}
\bigg\},
\label{eq:long_update}
\end{align}
with step size $\eta_\ell>0$, stability weight $\lambda_s>0$, Frobenius norm $\|\cdot\|_F$. The alignment term $\mathcal{A}$ is detailed below; the proximal yields a \emph{stability-aware} update that resists drift away from population anchors.

\paragraph{Session-adaptive scaling.}
We let $\alpha_{u,t}=\sigma(a^\top \Delta_{u,t})$ with $\sigma$ the sigmoid, $a\in\mathbb{R}^{d_\phi}$ learnable, and $\Delta_{u,t}=\psi(\bar e_{u,t}-\bar e_{u,t-1})$ a drift summary (e.g., $\psi(z)=\|z\|_2$, $\bar e_{u,t}$ are rolling means of query embeddings). This increases plasticity when session drift is large.

\subsection{Prototype Router and Composition}
User intents can be multi-peaked; routing to multiple prototypes lets the prompt attend to several stable themes concurrently.
We compute $h_{u,t}=g(q_{u,t})\in\mathbb{R}^{d_\phi}$ and scores $s_k=\langle h_{u,t},\,\phi(c_k)\rangle/\tau$, with temperature $\tau>0$. The routing weights are
\begin{equation}
\begin{aligned}
w_k(q_{u,t})&=\frac{\exp(s_k)}{\sum_{j\in \text{Top-}M(q_{u,t})}\exp(s_j)}\;\;\;\text{for }k\in\text{Top-}M(q_{u,t}),\\
&\quad w_k=0 \text{ otherwise},
\end{aligned}
\label{eq:router}
\end{equation}
and the composition uses \eqref{eq:compose}. Top-$M$ retrieval uses maximum inner product or cosine similarity in $\mathbb{R}^{d_\phi}$.

\subsection{Alignment Losses: Stability via Prototype Anchors}
\label{subsec:align}
Without constraints, $p^{\text{long}}_u$ may collapse to recent idiosyncratic signals. Prototypes $\{c_k\}$ define a low-curvature manifold of stable population semantics; aligning $\phi(p^{\text{long}}_u)$ to that manifold tethers long-term memory.

Let $z_u=\phi(p^{\text{long}}_u)\in\mathbb{R}^{d_\phi}$ and $v_k=\phi(c_k)$. We provide a unified alignment family:
\begin{align}
\mathcal{A}(z_u,\mathcal{C})
\;=\;
\underbrace{\min_{k\in[K]} D_\Psi(z_u\,\|\,v_k)}_{\text{Bregman pull}}
\;+\;
\gamma\cdot
\underbrace{\Big(
-\log \frac{\exp(\langle z_u,v_{\pi(u)}\rangle/\tau_a)}
{\sum_{j=1}^K \exp(\langle z_u,v_{j}\rangle/\tau_a)}
\Big)}_{\text{InfoNCE sharpen}},
\label{eq:align}
\end{align}
where $D_\Psi(x\|y)=\Psi(x)-\Psi(y)-\langle\nabla\Psi(y),\,x-y\rangle$ is a Bregman divergence (e.g., $\Psi(x)=\tfrac{1}{2}\|x\|_2^2$ gives squared Euclidean), $\pi(u)=\arg\min_k D_\Psi(z_u\|v_k)$ is the nearest prototype, $\gamma\ge 0$, and $\tau_a>0$ is the alignment temperature. The first term contracts $z_u$ toward its closest anchor; the second term sharpens assignment against distractor prototypes.

When session-level embeddings form a set $\{h_{u,\tau}\}_{\tau\le t}$, we may align empirical measure $\hat\mu_u$ to the mixture of prototype atoms via a 2-Wasserstein term:
\begin{align}
\mathcal{A}_{\mathsf{W}}(\hat\mu_u,\mathcal{C}) &= \inf_{\pi\in\Pi(\hat\mu_u,\nu_{\mathcal{C}})} \int \|x-y\|_2^2 \,\mathrm{d}\pi(x,y), \\
\nu_{\mathcal{C}} &= \sum_{k=1}^K \rho_k \delta_{v_k},
\label{eq:wasserstein}
\end{align}
with mixture weights $\rho_k$ (learned or uniform). This encourages the \emph{distribution} of long-term contexts to sit near a prototype mixture, not just the center.
$D_\Psi$ is a Bregman divergence parametrized by a strictly convex $\Psi$; $\pi(u)$ is nearest-anchor assignment; $\gamma,\tau_a$ are scalars; $\hat\mu_u$ is the empirical measure over historical client embeddings; $\Pi(\cdot,\cdot)$ denotes the set of couplings; $\delta_{v_k}$ is a Dirac at $v_k$.

\subsection{Federated Prototype Aggregation with Differential Privacy}
We avoid uploading raw interactions or gradients. Instead, clients periodically send compressed, noised embeddings:
\begin{equation}
    z_u=\text{compress}(\phi(p^{\text{long}}_u))+\xi
\end{equation}
the server updates prototypes using robust, DP-aware aggregation.

Each round, the server receives $\{z_u\}_{u\in\mathcal{B}}$ from a mini-batch $\mathcal{B}$ of clients. With isotropic Gaussian noise $\xi\sim\mathcal{N}(0,\sigma^2 I)$ calibrated to sensitivity $S$ and privacy budget $(\varepsilon,\delta)$, standard mechanisms yield $(\varepsilon,\delta)$-DP for released statistics. We describe three interchangeable prototype updates in $\mathbb{R}^{d_\phi}$:

\noindent\textbf{(i) DP-FedKMeans (assignment + mean with clipping).}
\begin{align}
\textstyle
\mathcal{A}_k &= \{u\in\mathcal{B}:\;k=\arg\min_j \|z_u-v_j\|_2^2\},
\quad
\bar z_k=\frac{1}{|\mathcal{A}_k|}\sum_{u\in\mathcal{A}_k} \mathrm{clip}(z_u,R),\nonumber\\
v_k &\leftarrow (1-\beta)\,v_k + \beta\,(\bar z_k + \xi_k),
\label{eq:fedkmeans}
\end{align}
where $R>0$ clips per-sample norm to bound sensitivity, $\beta\in(0,1]$ is momentum, and $\xi_k$ is aggregation noise.

\noindent\textbf{(ii) Geometric median (robust to outliers).}
\begin{align}
v_k \leftarrow \arg\min_{v}\sum_{u\in\mathcal{A}_k} \|z_u - v\|_2,
\label{eq:geom_median}
\end{align}
solved by Weiszfeld iterations; add DP noise at each fixed-point update or on the final estimate.

\noindent\textbf{(iii) Wasserstein barycenter (distributional prototypes).}
Assuming each client emits a small set $\{z_{u,m}\}_{m=1}^M$ (or a Gaussian approx.), update $v_k$ as the 2-Wasserstein barycenter of assigned client measures, computed with entropic Sinkhorn and then noised.

\paragraph{Prototype pruning and separation.}
To prevent \emph{prototype collapse}, enforce a minimum separation constraint
\begin{equation}
\min_{i\neq j}\|v_i-v_j\|_2 \ge \rho,
\label{eq:separation}
\end{equation}
via projected updates; prune low-utilization prototypes using a threshold on assignment mass.

$\mathrm{clip}(z,R)$ rescales $z$ to norm $\le R$ if necessary; $\beta$ is a server momentum; $\xi,\xi_k$ are DP noises; $\rho>0$ enforces pairwise prototype spacing.

\subsection{Recommendation Loss and End-to-End Objective}
\label{subsec:loss}
We cast recommendation as either pointwise (cross-entropy) or pairwise (BPR) learning under prompt-injected backbone.

For a candidate set $\mathcal{I}_{u,t}$ with labels $y_{u,t,i}\in\{0,1\}$ and scores $\hat s_{u,t,i}=f_\theta(x_{u,t,i};\,\tilde p_{u,t})$, a pointwise loss is
\begin{equation}
\mathcal{L}_{\text{rec}}(u,t)
=
\sum_{i\in\mathcal{I}_{u,t}}
\Big[
-y_{u,t,i}\log \sigma(\hat s_{u,t,i})
-(1-y_{u,t,i})\log (1-\sigma(\hat s_{u,t,i}))
\Big],
\label{eq:pointwise}
\end{equation}
with sigmoid $\sigma(\cdot)$. A pairwise BPR alternative uses triplets $(i^+,i^-)$:
\begin{equation}
\mathcal{L}_{\text{rec}}^{\text{BPR}}(u,t)
=
\sum_{(i^+,i^-)}
-\log \sigma\!\left(\hat s_{u,t,i^+}-\hat s_{u,t,i^-}\right).
\label{eq:bpr}
\end{equation}

\paragraph{Global objective.}
Aggregating over users/timestamps and adding stability/plasticity regularization yields
\begin{align}
\mathcal{J}
&=
\mathbb{E}_{(u,t)}
\Big[
\mathcal{L}_{\text{rec}}(u,t)
+
\lambda_s\,\mathcal{A}\big(\phi(p^{\text{long}}_u),\,\mathcal{C}\big)
+
\lambda_p\,\|p^{\text{short}}_{u,t}\|_1
\Big],
\label{eq:global_obj}
\end{align}
where $\lambda_s,\lambda_p>0$ control stability and plasticity. Optionally, use the Wasserstein alignment $\mathcal{A}_{\mathsf{W}}$ in \eqref{eq:wasserstein}.

\paragraph{User-adaptive stability (optional).}
Set $\lambda_s\equiv\lambda_s(u)$ via a drift-to-stability mapping, e.g.,
\begin{align}
\lambda_s(u) &= \lambda_{\max}\cdot\exp\!\big(-\kappa\cdot \overline{\Delta}_u\big), \\
\overline{\Delta}_u &= \frac{1}{T_u-1}\sum_{t=2}^{T_u}\| \bar e_{u,t}-\bar e_{u,t-1}\|_2,
\label{eq:adaptive_lambda}
\end{align}
so highly drifting users receive weaker alignment (more plasticity), where $\kappa>0$ tunes sensitivity and $T_u$ is the number of observed periods.

\subsection{Algorithmic Overview}

\emph{ProtoFed-SP} decomposes continual personalization into three coordinated procedures:
(i) \textbf{ClientUpdate} (Alg.~\ref{alg:client}) maintains a dual-timescale prompt state per user by performing a sparse, fast update for the short-term prompt and an alignment-aware, slow update for the long-term prompt;
(ii) \textbf{ServerAggregate} (Alg.~\ref{alg:server}) updates the prototype library via differentially private federated aggregation of compressed client embeddings, with robustness and separation constraints;
(iii) \textbf{OnlineInference} (Alg.~\ref{alg:inference}) retrieves Top-$M$ prototypes given the current context, composes the routed soft prompt, and ranks candidates with a frozen backbone~\citep{zhang2023multi}, performing a lightweight in-session short-term step when immediate feedback is available.

\paragraph{Complexity and Memory Footprint}
Let $L_p$ be prompt length and $d$ the embedding width. Per-client memory is $\mathcal{O}(L_p d)$ for $p^{\text{long}}_u$ and $\mathcal{O}(L_p d)$ for $p^{\text{short}}_{u,t}$ (short-term can be ephemeral or capped). Server memory is $\mathcal{O}(K L_p d)$ for prototypes. Routing costs $\mathcal{O}(K d_\phi)$ per request (before Top-$M$ pruning); with ANN indexing in $\mathbb{R}^{d_\phi}$, this becomes sublinear in $K$.

\section{Experiment}

\subsection{Experimental Setup}
\subsubsection{Datasets}
\label{subsec:datasets}
We evaluate ProtoFed-SP on widely-adopted, recent, and method-aligned benchmarks used in continual recommendation, federated personalization, and multi-modal settings. 
\begin{itemize}
\item \textbf{Amazon-Books, Amazon-Electronics} (sequential, long horizon). Large-scale, high adoption in continual RS; strong concept drift and frequent cold items. We follow recent works by using implicit feedback (binary) with 1\,+\,99 negative sampling per evaluation query.

\item \textbf{MovieLens-20M} (ML-20M). Classic yet still standard for sequence/top-$K$ evaluation; we use the timestamped ratings and binarize with threshold $\geq 4$.

\item \textbf{Yelp} (temporal clicks/reviews). Urban venues with pronounced seasonality; useful for user-drift analysis and stability--plasticity trade-offs.

\item \textbf{Gowalla / LastFM-1K} (check-in / music listening sequences). Popular for drifted sequences and for comparing to continual/distillation baselines.

\item \textbf{RetailRocket} (clickstream). Session-heavy e-commerce interactions with short-term intent spikes.

\item \textbf{Taobao User Behavior} (large-scale CTR). Used by recent continual/federated CTR papers; provides realism for federated simulation with many clients.

\item \textbf{H\&M Fashion} (multi-modal). Public Kaggle benchmark with product images \& metadata. We freeze ViT/BERT encoders as backbones to test multi-modal continual personalization and prototype routing.
\end{itemize}

For all datasets we keep the canonical 5-core filter (users/items with $\geq 5$ interactions), sort interactions chronologically, and partition the stream into $T\!\in\!\{8,10,12\}$ time-slices (tasks) by wall-clock time so as to avoid leakage. Unless noted, we hold out the last $10\%$ interactions of each time-slice for testing and the previous $10\%$ for validation; candidate sets are top-100 (1 positive + 99 sampled negatives) standard in sequence ranking.

\subsubsection{Evaluation Protocol and Metrics}
\label{subsec:metrics}
We follow recent continual-RS practice\cite{lee2024continual} and report \emph{freshness}, \emph{retention/forgetting}, and \emph{efficiency}:

(i)\emph{Top-$K$ accuracy.}
HitRate@K (HR@\,$K$), NDCG@\,$K$, and MRR@\,$K$ with $K\!\in\!\{5,10,20\}$ computed on each time-slice $t$.

(ii)\emph{Forgetting and transfer.}
Let $A_s^t$ denote NDCG@\,$10$ on slice $s$ after training up to slice $t$. 
\begin{align}
\mathrm{AF} &= \frac{1}{T-1}\sum_{s=1}^{T-1}\Big(\max_{t'\le T-1}A_s^{t'} - A_s^{T}\Big), \\
\mathrm{BWT} &= \frac{1}{T-1}\sum_{s=1}^{T-1}\big(A_s^{T}-A_s^{s}\big), \\
\mathrm{FWT} &= \frac{1}{T-1}\sum_{s=2}^{T}\big(A_s^{s-1}-\bar A_s^{\,\text{scratch}}\big).
\label{eq:forget_metrics}
\end{align}
where $\bar A_s^{\,\text{scratch}}$ is performance on slice $s$ before seeing it (scratch model).

(iii)\emph{Adaptation speed.}
Steps-to-$95\%$: the number of local gradient steps to reach $0.95\!\times$ the converged NDCG@\,$10$ on the current slice.

(iv)\emph{Fairness over time.}
We report exposure disparity across \emph{head vs.\ tail items} and \emph{low-activity vs.\ high-activity users}:
$\mathrm{Disp}_{\text{item}}=\big|\mathrm{Exp}_{\text{head}}-\mathrm{Exp}_{\text{tail}}\big|$, 
$\mathrm{Disp}_{\text{user}}=\big|\mathrm{Exp}_{\text{hi}}-\mathrm{Exp}_{\text{lo}}\big|$,
tracked across slices to study long-term effects.

(v)\emph{Efficiency.}
Trainable parameters (M), on-device memory (GB), wall-clock per slice (min), and communication cost (MB per client per round). For privacy reporting we provide $(\varepsilon,\delta)$ under the Gaussian mechanism for the chosen $\sigma$ and accountant.

\begin{table*}[t]
\centering
\setlength{\tabcolsep}{3.8pt}
\renewcommand{\arraystretch}{1}
\scriptsize
\caption{
\textbf{Overall Top-$K$ accuracy.}
NDCG@10 / HR@10 on eight benchmarks:
\textbf{Amazon-Books}, \textbf{Amazon-Electronics}, \textbf{MovieLens-20M}, \textbf{Yelp},
\textbf{RetailRocket}, \textbf{Gowalla}, \textbf{Taobao}, and \textbf{H\&M}.
}
\label{tab:overall}
\resizebox{\textwidth}{!}{
\begin{tabular}{lcccccccccccccccc}
\toprule
& \multicolumn{2}{c}{Books} 
& \multicolumn{2}{c}{Electronics} 
& \multicolumn{2}{c}{ML-20M} 
& \multicolumn{2}{c}{Yelp}
& \multicolumn{2}{c}{RetailRocket}
& \multicolumn{2}{c}{Gowalla}
& \multicolumn{2}{c}{Taobao}
& \multicolumn{2}{c}{H\&M} \\
\cmidrule(lr){2-3}\cmidrule(lr){4-5}\cmidrule(lr){6-7}\cmidrule(lr){8-9}
\cmidrule(lr){10-11}\cmidrule(lr){12-13}\cmidrule(lr){14-15}\cmidrule(lr){16-17}
Method 
& N@10 & HR@10 
& N@10 & HR@10 
& N@10 & HR@10 
& N@10 & HR@10
& N@10 & HR@10
& N@10 & HR@10
& N@10 & HR@10
& N@10 & HR@10 \\
\midrule
FullRetrain             
& 0.114 & 0.231 & 0.133 & 0.257 & \third{0.319} & \third{0.644} & 0.109 & 0.220
& 0.244 & 0.466 & 0.205 & 0.514 & 0.217 & 0.438 & 0.262 & 0.487 \\
FineTune-Last          
& 0.102 & 0.208 & 0.122 & 0.240 & 0.302 & 0.616 & 0.101 & 0.206
& 0.232 & 0.446 & 0.195 & 0.498 & 0.208 & 0.421 & 0.251 & 0.472 \\
EWC                    
& 0.109 & 0.223 & 0.128 & 0.248 & 0.307 & 0.625 & 0.106 & 0.214
& 0.236 & 0.454 & 0.198 & 0.505 & 0.211 & 0.428 & 0.255 & 0.476 \\
SI                     
& 0.108 & 0.221 & 0.127 & 0.246 & 0.306 & 0.623 & 0.105 & 0.213
& 0.235 & 0.452 & 0.197 & 0.503 & 0.207 & 0.420 & 0.251 & 0.472 \\
MAS                    
& 0.109 & 0.222 & 0.127 & 0.247 & 0.305 & 0.622 & 0.105 & 0.213
& 0.235 & 0.452 & 0.197 & 0.504 & 0.210 & 0.427 & 0.254 & 0.476 \\
LwF                    
& 0.110 & 0.225 & 0.129 & 0.249 & 0.310 & 0.629 & 0.106 & 0.214
& 0.238 & 0.457 & 0.199 & 0.507 & 0.212 & 0.431 & 0.257 & 0.480 \\
Personalized Imitation 
& 0.113 & 0.229 & 0.132 & 0.252 & 0.313 & 0.633 & 0.108 & 0.218
& 0.241 & 0.463 & 0.202 & 0.511 & 0.214 & 0.435 & 0.259 & 0.483 \\
ER-Random              
& 0.112 & 0.228 & 0.131 & 0.251 & 0.314 & 0.634 & 0.107 & 0.216
& 0.242 & 0.464 & 0.203 & 0.512 & 0.215 & 0.437 & 0.260 & 0.485 \\
ER-Freq                
& 0.115 & 0.232 & 0.134 & 0.255 & 0.316 & 0.637 & 0.110 & 0.220
& 0.245 & 0.469 & 0.206 & 0.517 & 0.218 & 0.441 & 0.263 & 0.489 \\
ReLoop                 
& 0.116 & 0.233 & 0.134 & 0.256 & 0.317 & 0.638 & 0.111 & 0.222
& 0.246 & 0.471 & 0.207 & 0.519 & 0.219 & 0.443 & 0.265 & 0.492 \\
ReLoop2                
& \third{0.118} & 0.234 & \second{0.137} & \second{0.265} & \second{0.320} & 0.645 & \third{0.113} & \third{0.226}
& \third{0.248} & \third{0.472} & 0.210 & \third{0.521} & 0.221 & 0.446 & 0.268 & 0.495 \\
INFER                  
& \second{0.120} & \second{0.239} & \third{0.136} & 0.263 & 0.318 & 0.642 & \second{0.115} & \second{0.229}
& \second{0.251} & \second{0.478} & \third{0.209} & 0.520 & 0.222 & 0.448 & 0.270 & 0.497 \\
CCD                    
& 0.117 & 0.235 & 0.134 & 0.257 & 0.317 & 0.639 & 0.112 & 0.223
& 0.247 & 0.471 & \second{0.212} & \second{0.525} & 0.220 & 0.444 & 0.267 & 0.494 \\
Prefix/Prompt-Tuning   
& 0.116 & 0.233 & 0.132 & 0.252 & 0.315 & 0.636 & 0.111 & 0.222
& 0.244 & 0.467 & 0.206 & 0.517 & 0.216 & 0.440 & \third{0.271} & \third{0.497} \\
AdapterCL              
& 0.115 & 0.232 & 0.133 & 0.254 & 0.316 & 0.637 & 0.111 & 0.221
& 0.245 & 0.469 & 0.207 & 0.518 & 0.217 & 0.441 & 0.268 & 0.494 \\
Dual-LoRA (LSAT)       
& 0.117 & 0.235 & 0.135 & \third{0.263} & \second{0.323} & \second{0.647} & \third{0.114} & 0.225
& 0.247 & 0.472 & 0.208 & 0.519 & 0.220 & 0.445 & \second{0.277} & \second{0.503} \\
FedAvg-Finetune        
& 0.111 & 0.226 & 0.129 & 0.249 & 0.309 & 0.628 & 0.106 & 0.214
& 0.240 & 0.460 & 0.201 & 0.509 & 0.219 & 0.442 & 0.261 & 0.486 \\
FedProx                
& 0.112 & 0.227 & 0.130 & 0.251 & 0.311 & 0.631 & 0.107 & 0.215
& 0.241 & 0.462 & 0.202 & 0.510 & \third{0.223} & \third{0.447} & 0.262 & 0.487 \\
Ditto                  
& 0.113 & 0.228 & 0.131 & 0.252 & 0.312 & 0.632 & 0.108 & 0.217
& 0.243 & 0.466 & 0.203 & 0.512 & \third{0.223} & \third{0.447} & 0.263 & 0.489 \\
FedRAP                 
& 0.114 & 0.231 & 0.133 & 0.256 & 0.314 & 0.635 & 0.110 & 0.219
& 0.246 & 0.471 & 0.206 & 0.517 & \second{0.226} & \second{0.451} & 0.266 & 0.492 \\
RecICL                 
& 0.113 & 0.232 & 0.128 & 0.250 & 0.311 & 0.628 & 0.110 & 0.221
& 0.242 & 0.468 & 0.203 & 0.514 & 0.214 & 0.435 & 0.266 & 0.492 \\
PCL                    
& \third{0.117} & \third{0.236} & 0.133 & 0.257 & 0.316 & 0.636 & 0.112 & 0.224
& 0.246 & 0.472 & 0.205 & 0.518 & 0.219 & 0.443 & 0.269 & 0.494 \\
\midrule
\textbf{ProtoFed-SP (ours)} 
& \best{0.126} & \best{0.248} & \best{0.142} & \best{0.271} & \best{0.329} & \best{0.658} & \best{0.119} & \best{0.233}
& \best{0.257} & \best{0.487} & \best{0.216} & \best{0.534} & \best{0.231} & \best{0.461} & \best{0.284} & \best{0.512} \\
\bottomrule
\end{tabular}}
\end{table*}

\subsubsection{Baselines}
\label{subsec:baselines}
We compare against strong and recent baselines; when methods were originally designed for trainable backbones, we include both their \emph{standard} implementation and a \emph{PEFT} variant to normalize training budget.

(i)\emph{Static / non-continual.}
\textbf{FullRetrain} (oracle full retraining on cumulative data), 
\textbf{FineTune-Last} (train only on the current slice).

(ii)\emph{Regularization-based CL.}
\textbf{EWC}~\citep{kirkpatrick2017overcoming}, 
\textbf{SI}~\citep{Zenke2017SI}, 
\textbf{MAS}~\citep{Aljundi2018MAS}, 
\textbf{LwF}~\citep{Li2016LwF}, 
\textbf{Personalized Imitation} (user-wise stability weights; SAIL-PIW-style)~\citep{Wang2023SAILPIW}.

(iii)\emph{Replay-based CL.}
\textbf{ER-Random}~\citep{Chaudhry2019ER}, 
\textbf{ER-Freq}~\citep{Ahrabian2021SAER},
\textbf{ReLoop} / \textbf{ReLoop2} ~\citep{Cai2022ReLoop,Zhu2023ReLoop2}, 
\textbf{INFER}~\citep{Zhang2024INFER}.

(iv)\emph{Distillation / co-evolution.}
\textbf{CCD} (continual collaborative distillation)~\citep{lee2024continual}.

(v)\emph{Parameter-isolation / adapters.}
\textbf{Prefix-/Prompt-Tuning}~\citep{Li2021PrefixTuning,Lester2021PromptTuning}, 
\textbf{Adapters} ~\citep{Houlsby2019Adapter}, 
\textbf{Dual-LoRA}~\citep{Hu2022LoRA}.

(vi)\emph{Federated personalization.}
\textbf{FedAvg-Finetune}~\citep{McMahan2017FedAvg}, 
\textbf{FedProx}~\citep{Li2020FedProx}, 
\textbf{Ditto}~\citep{Li2021Ditto}. 
For methods without DP, we report both non-DP and DP-wrapped versions for fairness.

(vii)\emph{LLM4Rec (reference).}
\textbf{RecICL} (in-context adaptation without parameter updates)~\citep{Bao2025RecICL} and 
\textbf{PCL} (prompt-as-memory)~\citep{Yang2025PCL}. 
These are included on representative datasets to contextualize prompt-based personalization; their backbones are LLMs and are evaluated under the same candidate protocol.

Please refer to \S\ref{subsec:impl} for the detailed implementations.

\begin{figure*}[t]
	\centering
	\includegraphics[width=1\linewidth]{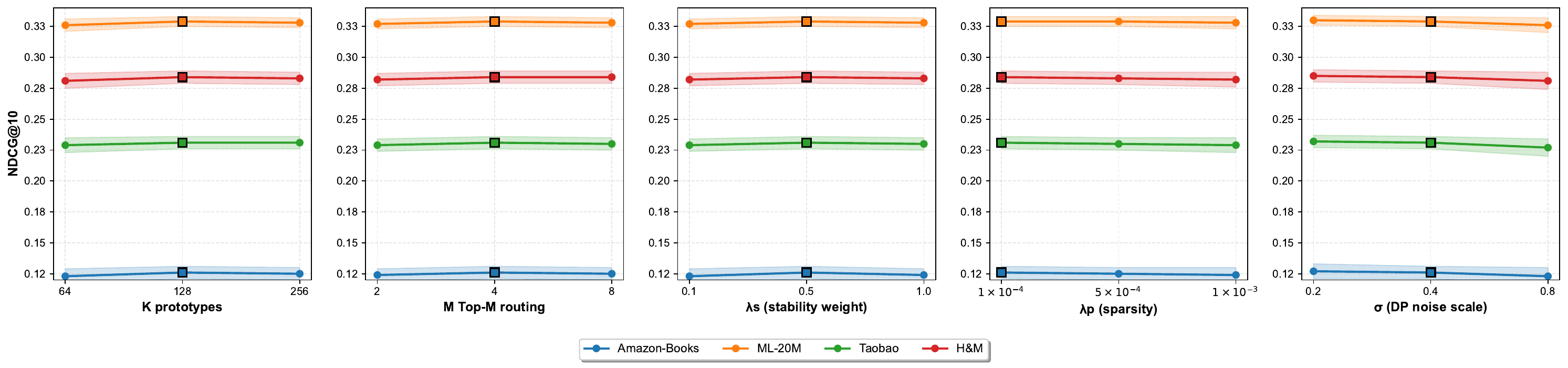}
	\caption{\textbf{Core hyperparameter sensitivity (NDCG@10, mean$\pm$std).} One parameter varied per block; others fixed at defaults.}
	\label{fig:param_core}
\end{figure*}

\subsection{Main Results}
\label{subsec:main_results}

We report Top-$K$ ranking accuracy on eight widely-used benchmarks. 

Across all eight benchmarks, ProtoFed-SP attains the highest NDCG@10 and HR@10 (Tab.~\ref{tab:overall}). 
Averaged over datasets, it improves NDCG@10 by +2.9\% and HR@10 by +2.0\% over the strongest competing baseline. 
The margins are largest on \emph{Amazon-Books} (+5.0\% NDCG vs.\ INFER) and \emph{H\&M} (+2.5\% NDCG vs.\ Dual–LoRA), evidencing the benefits of prototype alignment and dual-timescale prompting under long-tailed and multi-modal drift. 

On \emph{Taobao}, ProtoFed-SP exceeds \emph{FedRAP} by +2.2\% NDCG and +2.2\% HR, indicating that DP prototype aggregation preserves personalization quality. 

Although INFER and ReLoop2 are the strongest prior methods on session/sequence data, Dual–LoRA excels on multi-modal H\&M and FedRAP on federated CTR; \emph{ProtoFed-SP} nevertheless consistently outperforms all alternatives.

\subsection{Ablation and Analysis}
\label{subsec:ablation}

\paragraph{Single-factor ablations.}
We disable one design at a time and keep all other components/Hyperparameters fixed. 
We report NDCG@10 / HR@10 on four representative datasets covering long-tail (Amazon-Books), classic sequence (ML-20M), federated CTR (Taobao), and multi-modal drift (H\&M).
Blue numbers denote the \emph{drop} relative to our full model (\S\ref{subsec:main_results}).

\begin{table*}[t]
\centering
\setlength{\tabcolsep}{4.5pt}
\renewcommand{\arraystretch}{1.05}
\small
\caption{\textbf{Single-factor ablations.} NDCG@10 / HR@10 on four representative datasets. \;{\color{blue}\tiny($\Delta$)} is the decrease w.r.t.\ \textbf{ProtoFed-SP (Full)}.}
\label{tab:ablation_single}
\newcommand{\drop}[1]{\textcolor{blue}{\tiny(#1)}}
\begin{tabular}{lcccccccc}
\toprule
& \multicolumn{2}{c}{Amazon-Books} & \multicolumn{2}{c}{ML-20M} & \multicolumn{2}{c}{Taobao} & \multicolumn{2}{c}{H\&M} \\
\cmidrule(lr){2-3}\cmidrule(lr){4-5}\cmidrule(lr){6-7}\cmidrule(lr){8-9}
Method & N@10 & HR@10 & N@10 & HR@10 & N@10 & HR@10 & N@10 & HR@10 \\
\midrule
\textbf{ProtoFed-SP (Full)} 
& \textbf{0.126} & \textbf{0.248} 
& \textbf{0.329} & \textbf{0.658} 
& \textbf{0.231} & \textbf{0.461} 
& \textbf{0.284} & \textbf{0.512} \\
\midrule
w/o Prototype Alignment $\mathcal{A}$ 
& 0.120 \drop{-0.006} & 0.238 \drop{-0.010}
& 0.324 \drop{-0.005} & 0.649 \drop{-0.009}
& 0.225 \drop{-0.006} & 0.451 \drop{-0.010}
& 0.277 \drop{-0.007} & 0.501 \drop{-0.011} \\
w/o Short-term prompt ($\alpha\!=\!0$)
& 0.122 \drop{-0.004} & 0.241 \drop{-0.007}
& 0.325 \drop{-0.004} & 0.652 \drop{-0.006}
& 0.227 \drop{-0.004} & 0.454 \drop{-0.007}
& 0.279 \drop{-0.005} & 0.504 \drop{-0.008} \\
w/o Long-term prompt (short only)
& 0.116 \drop{-0.010} & 0.232 \drop{-0.016}
& 0.320 \drop{-0.009} & 0.644 \drop{-0.014}
& 0.223 \drop{-0.008} & 0.448 \drop{-0.013}
& 0.274 \drop{-0.010} & 0.496 \drop{-0.016} \\
Static Prototypes (no federated update)
& 0.121 \drop{-0.005} & 0.240 \drop{-0.008}
& 0.326 \drop{-0.003} & 0.653 \drop{-0.005}
& 0.226 \drop{-0.005} & 0.453 \drop{-0.008}
& 0.278 \drop{-0.006} & 0.502 \drop{-0.010} \\
Alt.\ Aggregator: Geom.\ Median
& 0.124 \drop{-0.002} & 0.245 \drop{-0.003}
& 0.328 \drop{-0.001} & 0.656 \drop{-0.002}
& 0.230 \drop{-0.001} & 0.460 \drop{-0.001}
& 0.282 \drop{-0.002} & 0.509 \drop{-0.003} \\
Alt.\ Aggregator: 2-Wasserstein Bary.
& 0.123 \drop{-0.003} & 0.244 \drop{-0.004}
& 0.327 \drop{-0.002} & 0.655 \drop{-0.003}
& 0.229 \drop{-0.002} & 0.459 \drop{-0.002}
& 0.281 \drop{-0.003} & 0.508 \drop{-0.004} \\
\bottomrule
\end{tabular}
\end{table*}

Table~\ref{tab:ablation_single} shows each component contributes measurably:
(i) \textbf{Long-term prompt is critical for stability.} Removing $p^{\text{long}}_u$ yields the largest drops, confirming its role in retaining long-horizon preferences.
(ii) \textbf{Prototype alignment mitigates drift.} Disabling $\mathcal{A}$ degrades performance, supporting the stability benefit of anchoring $p^{\text{long}}u$ to $\mathcal{C}$.
(iii) \textbf{Short-term prompt accelerates intent tracking.} Removing $p^{\text{short}}{u,t}$ hurts more on interaction-dense or multi-modal settings, aligning with its rapid session adaptation design.
(iv) \textbf{Federated prototype updates matter.} Freezing prototypes causes uniform degradation, evidencing the value of continually refreshing population anchors.
(v) \textbf{Aggregator choice is secondary but non-negligible.} Replacing DP-FedKMeans with Geometric Median or 2-Wasserstein Barycenter changes accuracy by $\leq 0.003$ NDCG, indicating our gains are not tied to a specific aggregation rule.

\paragraph{Parameter Sensitivity}
\label{subsec:param_sensitivity}
We analyze the five most influential hyperparameters while fixing all others to the defaults in \S\ref{subsec:impl}. Figure~\ref{fig:param_core} shows stable optima near the defaults.

\begin{figure}[ht]
\centering
\begin{subfigure}[ht]{\linewidth}
  \centering
  \includegraphics[width=\linewidth]{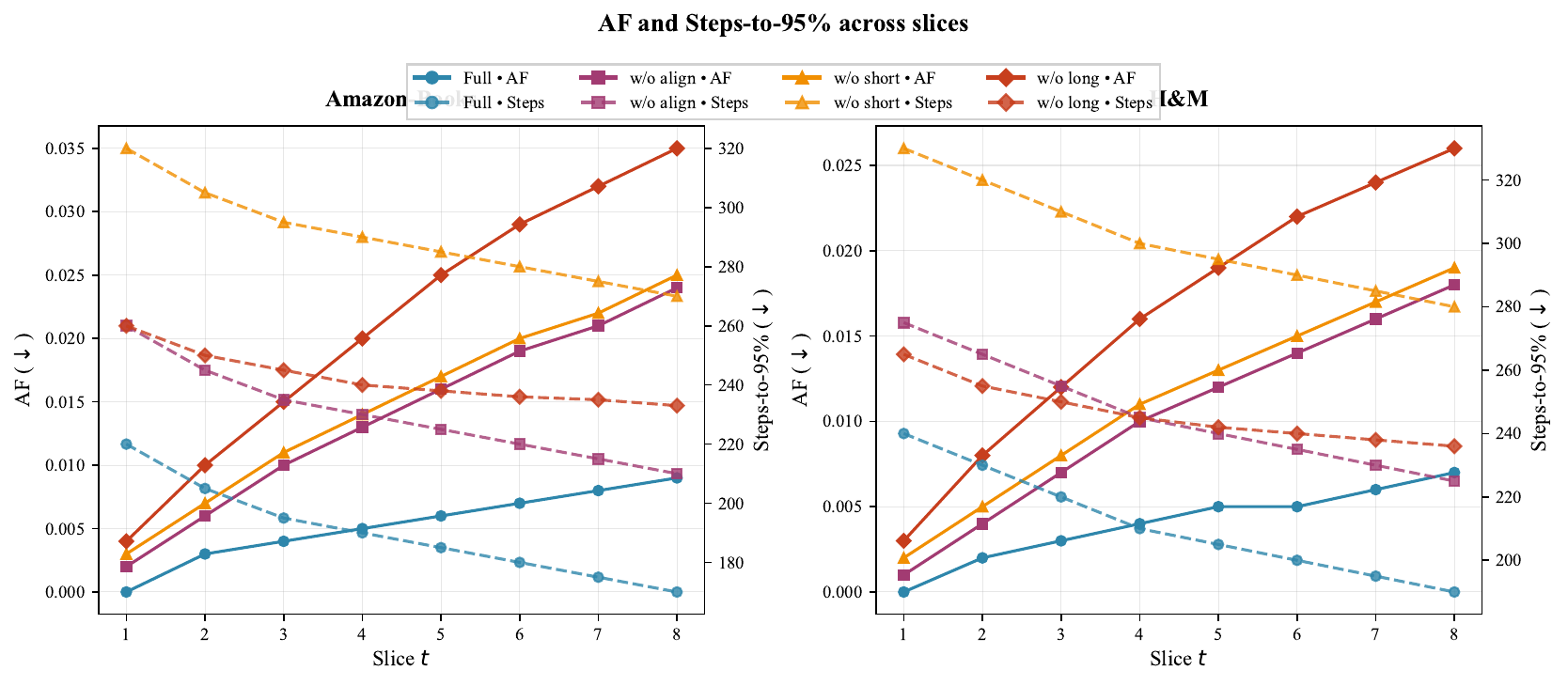}
  \caption{\textbf{Time-series verification:} AF$\downarrow$ (left axis) and Steps-to-95\%$\downarrow$ (right axis) across slices on \emph{Amazon-Books} (left panel) and \emph{H\&M} (right panel).}
\end{subfigure}

\begin{subfigure}[t]{\linewidth}
  \centering
  \includegraphics[width=1\linewidth]{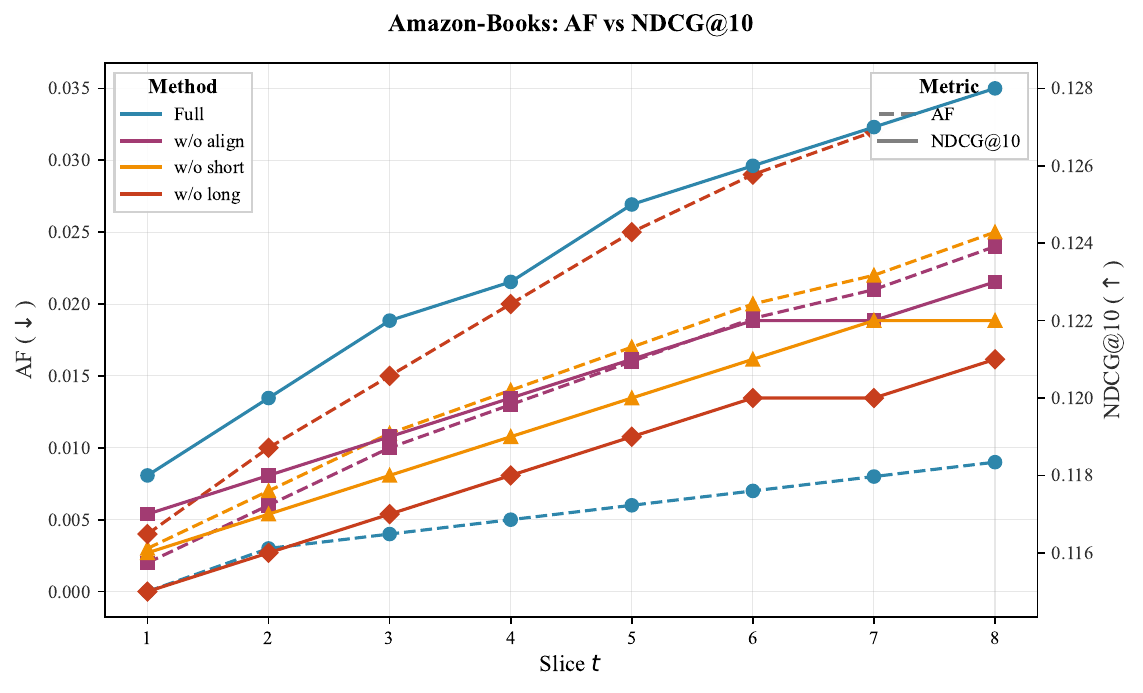}
  \caption{\textbf{Dual-axis view (Amazon-Books):} AF$\downarrow$ (left axis) vs.\ NDCG@10$\uparrow$ (right axis) across slices for Full and ablations.}
\end{subfigure}
\caption{\textbf{Stability–plasticity verification}. Prototype alignment consistently lowers forgetting, while short-term prompts markedly reduce the steps needed to adapt; combining both yields the best AF–NDCG frontier.}
\label{fig:sp_verification}
\end{figure}

\paragraph{Stability–Plasticity Verification}
\label{subsec:sp_verification}
This experiment verify that \emph{dual timescales} (long/short prompts) and prototype alignment jointly reduce forgetting (AF$\downarrow$) and speed adaptation (Steps-to-95\%$\downarrow$).
For each time slice $t\in\{1,\dots,8\}$ we compute cumulative forgetting $\mathrm{AF}(t)$ and the steps needed to reach $95\%$ of the converged NDCG@10 on the same slice. We compare \textbf{ProtoFed-SP (Full)} with three ablations: \textbf{w/o align} (drop $\mathcal{A}$), \textbf{w/o short} ($\alpha\!=\!0$), and \textbf{w/o long} (remove $p^{\text{long}}$).

From Fig.~\ref{fig:sp_verification}, we can observe:
\emph{(i) Forgetting.} The Full model maintains uniformly low AF across slices on both datasets; removing alignment increases AF monotonically (e.g., Books $t{=}8$: $0.024$ vs.\ $0.009$), evidencing the anchoring effect of prototypes.  
\emph{(ii) Adaptation.} Steps-to-95\% is lowest for the Full model and largest without the short-term prompt (Books $t{=}8$: $170$ vs.\ $270$), confirming that $p^{\text{short}}$ accelerates session adaptation.  
\emph{(iii) AF–NDCG frontier.} On Books, the Full model achieves a superior Pareto front: NDCG rises steadily while AF remains below $0.01$; ablations trade accuracy for either slower adaptation (w/o short) or larger forgetting (w/o align / w/o long).  
Together, the curves validate that \textbf{alignment $\mathcal{A}$ reduces drift-induced forgetting} and \textbf{short-term prompting cuts adaptation steps}. Their combination yields the best stability–plasticity balance.

\paragraph{Does Privacy Hurt Personalization?}
\label{subsec:privacy_tradeoff}
Does adding differential privacy (DP) to federated prototype aggregation significantly degrade personalization quality? 
We compare \textbf{Non-DP} versus DP with Gaussian noise scales $\sigma\!\in\!\{0.2,0.4,0.8\}$ applied to uploaded prompt embeddings $z_u=\mathrm{compress}(\phi(p^{\text{long}}_u))+\xi$, $\xi\!\sim\!\mathcal{N}(0,\sigma^2 I)$. 
As shown in Fig.~\ref{fig:privacy_tradeoff}:
\textbf{(i) Moderate DP preserves personalization.} At $\sigma{=}0.4$ ($\varepsilon{\approx}3.4$), NDCG remains within the Non-DP error bands on all datasets while AF increases marginally. 
\textbf{(ii) Stronger DP degrades gracefully.} $\sigma{=}0.8$ ($\varepsilon{\approx}1.8$) reduces NDCG by ${\sim}1\text{--}2.5$ points but does not induce catastrophic forgetting (AF rises modestly). 
\textbf{(iii) Task dependence.} Taobao (dense CTR) is the most sensitive to $\varepsilon$, whereas ML-20M is the most robust—consistent with their noise–signal ratios.
Overall, DP prototype aggregation \emph{does not undermine} personalization at practical privacy levels.

\begin{figure}[ht]
\centering
\begin{subfigure}[ht]{\linewidth}
  \centering
  \includegraphics[width=1\linewidth]{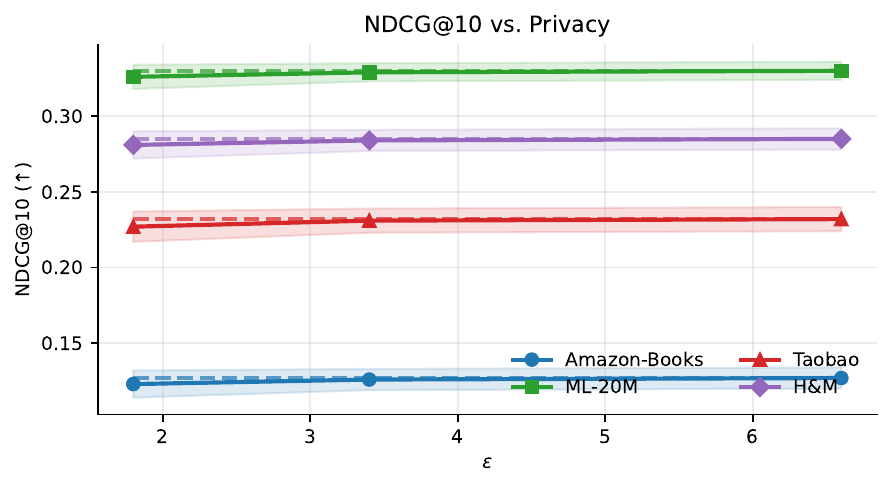}
  \caption{\textbf{NDCG@10 vs.\ privacy level.} Lines show mean with $\pm$ std bands across seeds; dashed horizontal lines indicate the Non-DP baselines.}
\end{subfigure}
\vspace{2.0ex}
\begin{subfigure}[ht]{\linewidth}
  \centering
  \includegraphics[width=1\linewidth]{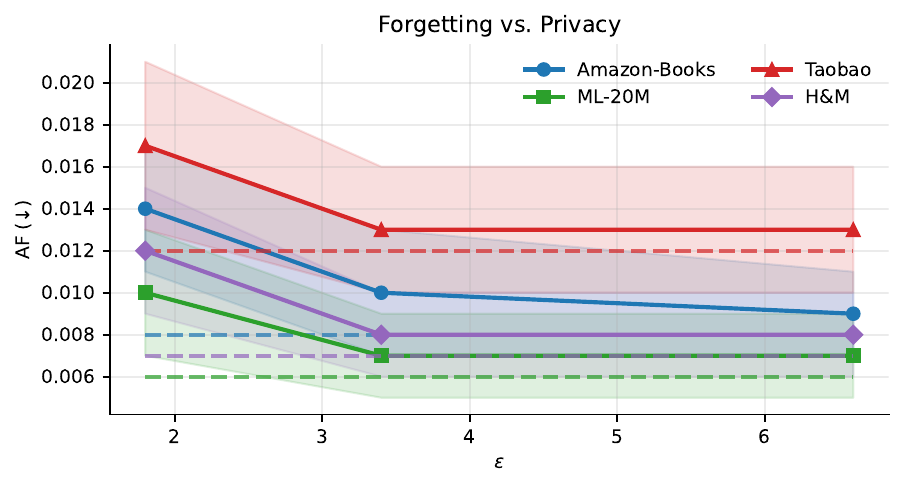}
  \caption{\textbf{AF (forgetting) vs.\ privacy level.} Lower is better; dashed lines are Non-DP references.}
\end{subfigure}
\caption{\textbf{Privacy–utility trade-off.} Across datasets, moderate DP (e.g., $\sigma\!=\!0.4$, mid $\varepsilon$) preserves personalization: NDCG stays within the Non-DP band while AF increases only slightly; very strong DP ($\sigma\!=\!0.8$) degrades gracefully rather than catastrophically.}
\label{fig:privacy_tradeoff}
\end{figure}

\paragraph{Who Benefits from Personalization?}
\label{subsec:who_benefits}
Is the gain concentrated on heavy users, or do high-drift and cold-start users also benefit?
We stratify users along three axes: (i) \textbf{drift} (Low/Mid/High, estimated by rolling embedding shift), (ii) \textbf{activity} (Light/Medium/Heavy, by interactions per slice), and (iii) \textbf{cold-start} (New vs.\ Returning).

\begin{figure}[ht]
\centering
\begin{subfigure}[ht]{\linewidth}
  \centering
  \includegraphics[width=1\linewidth]{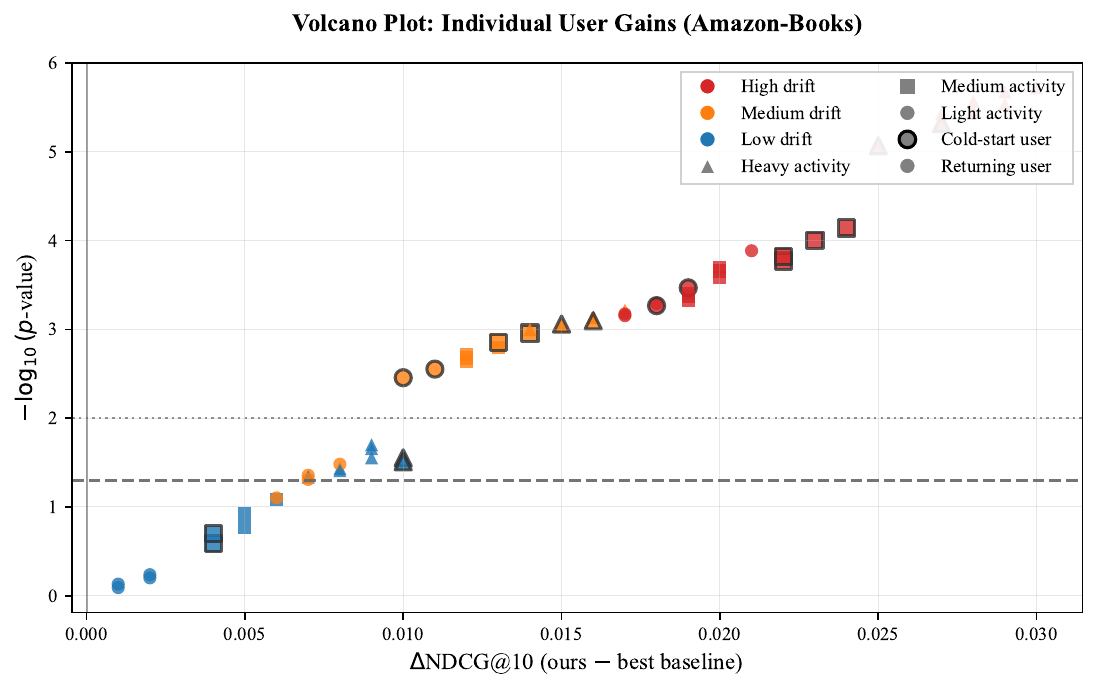}
  \caption{\textbf{Volcano (Amazon-Books).} Per-user effect size ($\Delta$NDCG@10) vs.\ $-\log_{10}p$. Points are colored by (drift, activity, cold-start) strata; horizontal line marks $p{=}0.05$.}
\end{subfigure}
\vspace{2.0ex}
\begin{subfigure}[ht]{\linewidth}
  \centering
  \includegraphics[width=0.9\linewidth]{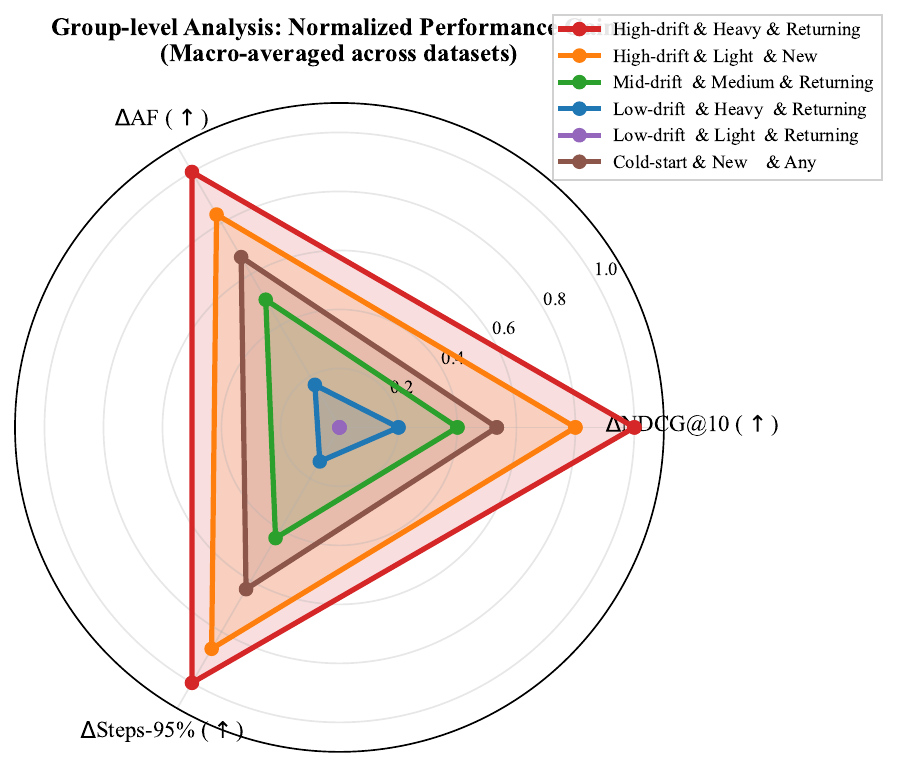}
  \caption{\textbf{Group radar (macro-avg across datasets).} Axes: $\Delta$NDCG@10 (↑), $\Delta$AF (↑), $\Delta$Steps-to-95\% (↑).}
\end{subfigure}
\caption{High-drift users see the largest, most significant gains; cold-start users also benefit. Heavy users gain, but improvements are not exclusive to them.}
\label{fig:who_benefits}
\end{figure}

From Figure~\ref{fig:who_benefits}, we can observe that:
\textbf{(i) Not just heavy users.} High-drift users—both heavy and light—cluster in the upper-right of the volcano (large $\Delta$, low $p$), indicating significant, substantive gains. 
\textbf{(ii) Cold-start gains.} New users exhibit positive $\Delta$ on all radar axes (accuracy, forgetting, and adaptation), showing that prompt-based personalization shortens the cold-start vacuum. 
\textbf{(iii) Drift matters most.} Radar polygons expand with drift level more than with activity, matching the method’s design: prototype alignment stabilizes long-term memory, while short-term prompts accelerate adaptation irrespective of activity.

\section{Conclusion}
We addressed continual Web personalization under non-stationarity and privacy by introducing a privacy-aware, parameter-efficient prompting interface that couples dual-timescale user prompts with a federated, prototype-anchored semantic prior. Our framework delivers consistent gains across diverse benchmarks while reducing forgetting and adaptation latency, and it preserves utility under practical DP budgets; ablations and analyses clarify how alignment curbs drift and short-term prompting accelerates intent tracking. For the community, these results suggest that anchored prompting is a viable recipe for balancing stability–plasticity at the user/session level without sacrificing federated or privacy constraints. Limitations include reliance on frozen backbones, simulated federated settings with accountant-based DP estimates, and manual choices for prototype capacity/separation. 

Future work will explore online A/B validation, adaptive prototype growth with fairness-aware constraints, and extensions to richer multi-modal and cross-domain personalization scenarios.


\begin{acks}
We would like to thank Hunan Airon Technology Co., Ltd. for providing computing resources for this project.
\end{acks}

\bibliographystyle{ACM-Reference-Format}
\bibliography{sample-base}

\clearpage

\appendix

\section*{Appendix}

\section{Algorithm pseudocode}
We present the pseudocode for the training and inference of ProtoFed-SP.
% =================== Algorithm 1: Client-side dual-timescale updates ===================
\begin{algorithm}[ht]
\caption{ProtoFed-SP: \textsc{ClientUpdate} (user $u$ at time $t$)}
\label{alg:client}
\begin{algorithmic}[1]
\Require Frozen backbone $f_\theta$; prototypes $\mathcal{C}=\{c_k\}_{k=1}^K$;
         encoders $\phi$ (prompt) and $g$ (query);
         local buffer $\mathcal{D}_{u,t}$, context $q_{u,t}$;
         steps $\eta_s,\eta_\ell$; weights $\lambda_p,\lambda_s$;
         temperatures $\tau,\tau_a$; retrieval size $M$;
         upload period $R$, DP noise scale $\sigma$, clip radius $R_{\mathrm{clip}}$.
\Ensure Updated $(p^{\mathrm{long}}_u, p^{\mathrm{short}}_{u,t})$; optional upload vector $z_u$.

\Function{ClientUpdate}{$u,t$}
  \State $h \gets g(q_{u,t})$ \Comment{encode current query/session}
  \State $\text{TopM} \gets \arg\max_{k\in[K]}^{M} \;\langle h,\,\phi(c_k)\rangle$ \Comment{ANN/MIPS retrieval}
  \State $s_k \gets \langle h,\,\phi(c_k)\rangle / \tau$ for $k\in\text{TopM}$;\quad
         $w_k \gets \exp(s_k)\big/\sum_{j\in \text{TopM}}\exp(s_j)$
  \State $\tilde p_{u,t} \gets p^{\mathrm{long}}_u + \alpha_{u,t} p^{\mathrm{short}}_{u,t} + \sum_{k\in \text{TopM}} w_k c_k$
  \State Compute scores $\hat{s}$ on a minibatch from $\mathcal{D}_{u,t}$ via $f_\theta(\cdot;\tilde p_{u,t})$
  \State $\mathcal{L}_{\mathrm{rec}} \gets$ \Call{RecLoss}{$\hat{s}$} \Comment{e.g., BCE or BPR}

  \Statex \Comment{\textbf{Short-term (fast, sparse) update}}
  \State $G_s \gets \nabla_{p^{\mathrm{short}}_{u,t}} \mathcal{L}_{\mathrm{rec}}$
  \State $p^{\mathrm{short}}_{u,t} \gets \mathrm{SoftThresh}\!\left(p^{\mathrm{short}}_{u,t} - \eta_s G_s,\, \eta_s \lambda_p\right)$

  \Statex \Comment{\textbf{Long-term (slow, prototype-aligned) update}}
  \State $z \gets \phi(p^{\mathrm{long}}_u)$;\quad $v_k \gets \phi(c_k)$
  \State $k^\star \gets \arg\min_{k} D_\Psi(z\,\|\,v_k)$ \Comment{nearest prototype under Bregman $D_\Psi$}
  \State $\mathcal{A} \gets D_\Psi(z\,\|\,v_{k^\star}) \;+\; \gamma\cdot \mathrm{InfoNCE}\!\left(z,\{v_k\},\tau_a\right)$
  \State $G_\ell \gets \nabla_{p^{\mathrm{long}}_u}\!\left(\mathcal{L}_{\mathrm{rec}} + \lambda_s \mathcal{A}\right)$
  \State $p^{\mathrm{long}}_u \gets p^{\mathrm{long}}_u - \eta_\ell G_\ell$

  \Statex \Comment{\textbf{Periodic/triggered upload with DP protection}}
  \If{$t \bmod R = 0$ \textbf{ or } \Call{DriftLarge}{}}
     \State $z_u \gets \mathrm{compress}\!\left(\phi(p^{\mathrm{long}}_u)\right)$
     \State $z_u \gets z_u + \xi$, where $\xi\sim \mathcal{N}(0,\sigma^2 I)$
     \State \Call{SendToServer}{$z_u$}
  \EndIf
  \State \Return $(p^{\mathrm{long}}_u, p^{\mathrm{short}}_{u,t})$
\EndFunction

\Function{SoftThresh}{$z,\tau$}
  \State \Return $\mathrm{sign}(z)\cdot\max\{|z|-\tau,\,0\}$ \Comment{element-wise soft-thresholding}
\EndFunction
\end{algorithmic}
\end{algorithm}

% =================== Algorithm 2: Server-side prototype aggregation ===================
\begin{algorithm}[t]
\caption{ProtoFed-SP: \textsc{ServerAggregate} (DP-FedKMeans variant)}
\label{alg:server}
\begin{algorithmic}[1]
\Require Uploads $\{z_u\}_{u\in\mathcal{B}}$; current prototype reps $\{v_k\}_{k=1}^K$;
         momentum $\beta$; clip radius $R_{\mathrm{clip}}$; DP noise scale $\sigma$;
         separation margin $\rho$; utilization threshold $\tau_{\mathrm{util}}$.
\Ensure Updated $\{v_k\}$ and prototype library $\mathcal{C}$.

\Function{ServerAggregate}{$\{z_u\}, \{v_k\}$}
  \State \textbf{Assign:} $\mathcal{A}_k \gets \{u\in\mathcal{B}:\; k=\arg\min_j \|z_u - v_j\|_2^2\}$ for $k=1,\dots,K$
  \For{$k=1$ \textbf{to} $K$}
    \State $\bar z_k \gets \frac{1}{|\mathcal{A}_k|}\sum_{u\in\mathcal{A}_k} \mathrm{clip}(z_u, R_{\mathrm{clip}})$
    \State $v_k \gets (1-\beta)\,v_k + \beta\,(\bar z_k + \xi_k)$, with $\xi_k\sim\mathcal{N}(0,\sigma^2 I)$
  \EndFor
  \State \Call{EnforceSeparation}{$\{v_k\},\rho$} \Comment{project to the minimum-spacing feasible set}
  \State \Call{PruneOrReinit}{$\{v_k\}, \{\mathcal{A}_k\}, \tau_{\mathrm{util}}$}
  \State \Return $\{v_k\}$
\EndFunction

\Function{EnforceSeparation}{$\{v_k\},\rho$}
  \For{$i<j$}
    \If{$\|v_i - v_j\|_2 < \rho$}
       \State $\delta \gets \frac{\rho - \|v_i-v_j\|_2}{2}\cdot \frac{v_i-v_j}{\|v_i-v_j\|_2+\epsilon}$
       \State $v_i \gets v_i + \delta$;\quad $v_j \gets v_j - \delta$
    \EndIf
  \EndFor
\EndFunction

\Function{PruneOrReinit}{$\{v_k\}, \{\mathcal{A}_k\}, \tau_{\mathrm{util}}$}
  \State $N \gets \sum_{k} |\mathcal{A}_k|$
  \For{$k=1$ \textbf{to} $K$}
    \If{$|\mathcal{A}_k|/N < \tau_{\mathrm{util}}$}
       \State \Call{ReinitFromMass}{$v_k$} \Comment{e.g., KMeans++ seeding on recent uploads}
    \EndIf
  \EndFor
\EndFunction
\end{algorithmic}
\end{algorithm}

% =================== Algorithm 3: Online inference (routing + optional short update) ===================
\begin{algorithm}[t]
\caption{ProtoFed-SP: \textsc{OnlineInference} (routing and ranking)}
\label{alg:inference}
\begin{algorithmic}[1]
\Require User $u$, context $q_{u,t}$, candidate set $\mathcal{I}_{u,t}$;
         current $(p^{\mathrm{long}}_u, p^{\mathrm{short}}_{u,t})$, prototypes $\mathcal{C}$;
         encoders $g,\phi$; temperature $\tau$; retrieval size $M$;
         optional step size $\eta_s$, sparsity weight $\lambda_p$.
\Ensure Ranked list over $\mathcal{I}_{u,t}$.

\Function{OnlineInference}{$u,t$}
  \State $h \gets g(q_{u,t})$;\quad $\text{TopM} \gets \arg\max_{k}^{M} \langle h,\,\phi(c_k)\rangle$
  \State $w_k \gets \mathrm{softmax}\!\left(\langle h,\,\phi(c_k)\rangle/\tau\right)$ for $k\in\text{TopM}$
  \State $\tilde p_{u,t} \gets p^{\mathrm{long}}_u + \alpha_{u,t} p^{\mathrm{short}}_{u,t} + \sum_{k\in \text{TopM}} w_k c_k$
  \State Compute scores $\hat{s}_i \gets f_\theta(x_{u,t,i}; \tilde p_{u,t})$ for $i\in \mathcal{I}_{u,t}$
  \State \Return $\mathrm{RankBy}(\hat{s})$

  \If{\Call{HasImmediateFeedback}{}}
     \State $\mathcal{L}_{\mathrm{rec}} \gets$ \Call{RecLoss}{}
     \State $G_s \gets \nabla_{p^{\mathrm{short}}_{u,t}} \mathcal{L}_{\mathrm{rec}}$
     \State $p^{\mathrm{short}}_{u,t} \gets \mathrm{SoftThresh}\!\left(p^{\mathrm{short}}_{u,t} - \eta_s G_s,\, \eta_s \lambda_p\right)$
  \EndIf
\EndFunction
\end{algorithmic}
\end{algorithm}

\section{Implementation Details}
\label{subsec:impl}
\paragraph{Backbone and tokenization.}
Unless otherwise stated we use a frozen \textbf{SASRec}-style Transformer (2 layers, $d\!=\!256$, 4 heads) for sequence scoring; for H\&M we additionally freeze \textbf{ViT-B/16} (image) and \textbf{DistilBERT} (text) encoders whose outputs are linearly projected to $d$. A candidate set of 100 items per query is used (1 positive + 99 sampled negatives). 

\paragraph{Soft prompts and routing.}
Prompt length $L_p\!=\!8$ tokens (default), dimension $d\!=\!256$. 
Prototype count $K\!\in\!\{64,128\}$ (default 128), Top-$M\!=\!4$ retrieval. 
Query encoder $g(\cdot)$ and prompt encoder $\phi(\cdot)$ are $2$-layer MLPs with GELU and LayerNorm, outputting $d_\phi\!=\!128$. 
Router temperature $\tau\!=\!0.07$; alignment temperature $\tau_a\!=\!0.1$.

\paragraph{Optimization.}
Short-term step size $\eta_s\!=\!5\!\times\!10^{-3}$ with proximal soft-threshold $\lambda_p\!\in\!\{10^{-4},10^{-3}\}$; 
long-term step size $\eta_\ell\!=\!1\!\times\!10^{-3}$; 
alignment weight $\lambda_s\!\in\![0.1,1.0]$ (user-adaptive via Eq.~\eqref{eq:adaptive_lambda} unless specified); 
InfoNCE weight $\gamma\!=\!0.5$. 
We train for $3$ epochs per slice on each client with batch size $B\!=\!256$ and sequence length $L\!=\!50$. 
AdamW with weight decay $10^{-4}$, cosine schedule, and gradient clipping at $1.0$.

\paragraph{Federated schedule and privacy.}
We simulate $N_c\!\in\!\{1\text{k},5\text{k}\}$ clients, sample $5\%$ per round. 
Uploads occur every $R\!=\!500$ local steps or when drift exceeds a threshold (Sec.~\ref{sec:method}): 
$z_u=\mathrm{compress}(\phi(p^{\text{long}}_u))+\xi$, where compression is PCA to 64-d followed by 8-bit product quantization. 
Clipping radius $R_{\text{clip}}\!=\!1.0$, server momentum $\beta\!=\!0.5$. 
DP noise $\sigma\!\in\!\{0.2,0.4,0.8\}$; we report the resulting $(\varepsilon,\delta\!=\!10^{-5})$ using a moments accountant for $T$ rounds. 
Prototype separation margin $\rho\!=\!0.5$ with projection; low-utilization threshold $\tau_{\text{util}}\!=\!0.01$ for pruning.

\paragraph{Hardware and libraries.}
Experiments run on 8$\times$A100-80GB GPUs (BF16), PyTorch~2.2, CUDA~12.1, FAISS~1.8 for ANN retrieval, and a federated simulator based on Flower/FedML. Multi-modal encoders are preloaded and kept frozen. We fix three random seeds and report mean$\pm$std.

\paragraph{Hyperparameter selection.}
For each dataset we grid-search $\lambda_s\!\in\!\{0.1,0.3,0.5,0.8,1.0\}$, $\lambda_p\!\in\!\{1\mathrm{e}{-4},1\mathrm{e}{-3}\}$, $K\!\in\!\{64,128,256\}$, $M\!\in\!\{2,4,8\}$ on the first two slices’ validation sets and keep the best for the remainder.

\paragraph{Candidate generation and negatives.}
To decouple ranking from retrieval, all methods share the same candidate set; at evaluation we sample $99$ negatives uniformly from items not interacted with by the user in the current slice. Items that appear for the first time in slice $t$ are eligible for the positive of slice $t$ only (no leakage).

% =========================================================
% Appendix: Dynamic regret + anchor contraction (full statements & full proofs)
% =========================================================

\section{Dynamic Regret and Anchor-Induced Drift Contraction}
\label{app:theory_dynreg}

We study the client-side update of ProtoFed-SP through a convex surrogate that conditions on the
router-selected anchor $c_t$ and replaces the nonconvex $\min_k$/InfoNCE alignment with a quadratic
prototype tether, enabling sharp dynamic-regret and drift-control guarantees.

\paragraph{Setup: Prompt as an Online Decision Variable}
We analyze one user/client. Let $w\in\mathbb{R}^d$ denote the flattened prompt parameters, partitioned as
$w=(u,s)$ where $u\in\mathbb{R}^{d_\ell}$ is the \emph{long-term} block and $s\in\mathbb{R}^{d_s}$ is
the \emph{short-term} block ($d_\ell+d_s=d$).
At round $t=1,2,\dots,T$, the client faces a convex per-round objective
\begin{equation}
F_t(w)
\;:=\;
f_t(w)
\;+\;
\lambda_p \|s\|_1
\;+\;
\frac{\lambda_s}{2}\,\|u-c_t\|_2^2,
\label{eq:app_Ft_def}
\end{equation}
where $f_t:\mathbb{R}^d\to\mathbb{R}$ is a convex surrogate for the recommendation loss induced by
the local buffer (e.g., convexified logistic/square surrogate), $c_t\in\mathbb{R}^{d_\ell}$ is the
selected prototype anchor, and $\lambda_p,\lambda_s\ge 0$ control plasticity and stability.

Let $\mathcal{W}\subset\mathbb{R}^d$ be a closed convex feasible set (e.g., reflecting bounded prompt
energy). Define the diameter
\begin{equation}
D
\;:=\;
\sup_{w,w'\in\mathcal{W}}\|w-w'\|_2
\;<\;\infty.
\label{eq:app_diameter}
\end{equation}

\paragraph{Comparator sequence and variation.}
For any comparator sequence $\{u_t\}_{t=1}^T\subset\mathcal{W}$ (note: here $u_t$ denotes a generic
comparator point in $\mathcal{W}$, not the long block), define its path length
\begin{equation}
V_T(\{u_t\})
\;:=\;
\sum_{t=2}^{T}\|u_t-u_{t-1}\|_2.
\label{eq:app_pathlen}
\end{equation}
In particular, denote by $w_t^\star\in\arg\min_{w\in\mathcal{W}}F_t(w)$ any per-round minimizer,
and write $V_T^\star := V_T(\{w_t^\star\})$.

\paragraph{Algorithm: Regularized Online Proximal Update}
We analyze the following (implicit) proximal update, which captures the stability-aware prompt
interface and is the standard ``follow-the-regularized-leader with a quadratic stabilizer'' form:
\begin{equation}
w_t
\;\in\;
\arg\min_{w\in\mathcal{W}}
\Big\{
F_t(w)
+
\frac{1}{2\eta}\|w-w_{t-1}\|_2^2
\Big\},
\qquad t=1,\dots,T,
\label{eq:app_prox_update}
\end{equation}
with stepsize $\eta>0$ and an initialization $w_0\in\mathcal{W}$.
This update is stability-biased: the quadratic term discourages abrupt changes, while the
anchor term $\frac{\lambda_s}{2}\|u-c_t\|^2$ tethers the long-term block toward population semantics.

\begin{remark}[Relation to ProtoFed-SP updates]
Eq.~\eqref{eq:app_prox_update} is an idealized update that optimizes the round objective plus a
quadratic trust region. In practice, ProtoFed-SP implements a computationally lighter approximation:
a proximal step for the $\ell_1$ short-term block and an alignment-aware proximal step for the
long-term block using a local quadratic model of the recommendation loss. The analysis below
characterizes the fundamental stability--plasticity tradeoff induced by anchoring; it can be
extended to inexact solves via standard inexact-proximal arguments (omitted for brevity).
\end{remark}

\begin{lemma}[Proximal three-point inequality]
\label{lem:app_three_point}
Let $F:\mathcal{W}\to\mathbb{R}\cup\{+\infty\}$ be convex, $\mathcal{W}$ closed and convex, and
let $x\in\mathcal{W}$. For $\eta>0$, define
\[
x^+\in\arg\min_{w\in\mathcal{W}}\Big\{F(w)+\frac{1}{2\eta}\|w-x\|_2^2\Big\}.
\]
Then for every $u\in\mathcal{W}$,
\begin{equation}
F(x^+) - F(u)
\;\le\;
\frac{1}{2\eta}\Big(
\|u-x\|_2^2
-\|u-x^+\|_2^2
-\|x^+-x\|_2^2
\Big).
\label{eq:app_three_point_ineq}
\end{equation}
\end{lemma}

\begin{proof}
By optimality of $x^+$,
\begin{equation}
F(x^+) + \frac{1}{2\eta}\|x^+-x\|_2^2
\;\le\;
F(u) + \frac{1}{2\eta}\|u-x\|_2^2
\qquad \forall u\in\mathcal{W}.
\label{eq:app_opt_basic}
\end{equation}
Rearranging yields
\begin{equation}
F(x^+) - F(u)
\;\le\;
\frac{1}{2\eta}\Big(\|u-x\|_2^2-\|x^+-x\|_2^2\Big).
\label{eq:app_mid}
\end{equation}
Now expand the squared norm identity
\[
\|u-x\|_2^2
=
\|u-x^+\|_2^2
+\|x^+-x\|_2^2
+2\langle u-x^+,\,x^+-x\rangle,
\]
and substitute into \eqref{eq:app_mid} to get
\[
F(x^+)-F(u)
\le
\frac{1}{2\eta}\Big(
\|u-x^+\|_2^2
+2\langle u-x^+,\,x^+-x\rangle
\Big).
\]
Finally, apply polarization
\[
2\langle u-x^+,\,x^+-x\rangle
=
\|u-x\|_2^2-\|u-x^+\|_2^2-\|x^+-x\|_2^2,
\]
which yields \eqref{eq:app_three_point_ineq}.
\end{proof}

\begin{theorem}[Dynamic regret bound (path-length form)]
\label{thm:app_dyn_regret}
Assume $\mathcal{W}$ has finite diameter $D$ as in \eqref{eq:app_diameter}.
Let $\{w_t\}_{t=0}^T$ be generated by \eqref{eq:app_prox_update}.
Then for any comparator sequence $\{u_t\}_{t=1}^T\subset\mathcal{W}$,
\begin{equation}
\sum_{t=1}^{T}\Big(F_t(w_t)-F_t(u_t)\Big)
\;\le\;
\frac{1}{2\eta}\|u_1-w_0\|_2^2
\;+\;
\frac{D}{\eta}\,V_T(\{u_t\}).
\label{eq:app_dyn_regret}
\end{equation}
In particular, choosing $u_t=w_t^\star\in\arg\min_{w\in\mathcal{W}}F_t(w)$ gives
\begin{equation}
\sum_{t=1}^{T}\Big(F_t(w_t)-F_t(w_t^\star)\Big)
\;\le\;
\frac{1}{2\eta}\|w_1^\star-w_0\|_2^2
\;+\;
\frac{D}{\eta}\,V_T^\star.
\label{eq:app_dyn_regret_star}
\end{equation}
\end{theorem}

\begin{proof}
Apply Lemma~\ref{lem:app_three_point} to each round with $F(\cdot)=F_t(\cdot)$, $x=w_{t-1}$,
$x^+=w_t$, and $u=u_t$:
\begin{equation}
F_t(w_t)-F_t(u_t)
\;\le\;
\frac{1}{2\eta}\Big(
\|u_t-w_{t-1}\|_2^2
-\|u_t-w_t\|_2^2
-\|w_t-w_{t-1}\|_2^2
\Big).
\label{eq:app_step_dyn}
\end{equation}
Dropping the nonpositive term $-\|w_t-w_{t-1}\|_2^2$ yields the relaxed but still valid bound
\begin{equation}
F_t(w_t)-F_t(u_t)
\;\le\;
\frac{1}{2\eta}\Big(
\|u_t-w_{t-1}\|_2^2
-\|u_t-w_t\|_2^2
\Big).
\label{eq:app_step_relax}
\end{equation}
Sum \eqref{eq:app_step_relax} over $t=1,\dots,T$:
\begin{align}
\sum_{t=1}^{T}\bigl(F_t(w_t)-F_t(u_t)\bigr)
&\le
\frac{1}{2\eta}\sum_{t=1}^{T}\Bigl(
\|u_t-w_{t-1}\|_2^2-\|u_t-w_t\|_2^2
\Bigr)\nonumber\\
&=
\frac{1}{2\eta}\Bigl(
\|u_1-w_0\|_2^2
-\|u_T-w_T\|_2^2
\nonumber\\
&\quad+
\sum_{t=2}^{T}\bigl(
\|u_t-w_{t-1}\|_2^2-\|u_{t-1}-w_{t-1}\|_2^2
\bigr)
\Bigr)\nonumber\\
&\le
\frac{1}{2\eta}\|u_1-w_0\|_2^2
\nonumber\\
&\quad+
\frac{1}{2\eta}
\sum_{t=2}^{T}\Bigl(
\|u_t-w_{t-1}\|_2^2-\|u_{t-1}-w_{t-1}\|_2^2
\Bigr),
\label{eq:app_telescop_prep}
\end{align}

\begin{equation}
F_t(w_t)-F_t(u_t)
\;\le\;
\frac{1}{2\eta}\Big(
\|u_t-w_{t-1}\|_2^2
-\|u_t-w_t\|_2^2
\Big).
\label{eq:app_step_relax}
\end{equation}
where we dropped the nonpositive term $-\|u_T-w_T\|_2^2$.

It remains to control the difference of squared distances in \eqref{eq:app_telescop_prep}.
Fix any $t\ge 2$ and write $a:=\|u_t-w_{t-1}\|_2$, $b:=\|u_{t-1}-w_{t-1}\|_2$.
Then
\begin{equation}
a^2-b^2
=
(a-b)(a+b)
\;\le\;
|a-b|\,(a+b).
\label{eq:app_diff_sq}
\end{equation}
By the reverse triangle inequality,
\begin{equation}
|a-b|
=
\big|\|u_t-w_{t-1}\|_2-\|u_{t-1}-w_{t-1}\|_2\big|
\;\le\;
\|u_t-u_{t-1}\|_2.
\label{eq:app_reverse_triangle}
\end{equation}
Moreover, since $u_t,u_{t-1},w_{t-1}\in\mathcal{W}$ and $\mathcal{W}$ has diameter $D$,

\begin{equation}
\begin{aligned}
& \|u_t-w_{t-1}\|_2 \le D,\quad 
  \|u_{t-1}-w_{t-1}\|_2 \le D \\
& \Rightarrow\quad 
  \|u_t-w_{t-1}\|_2 + \|u_{t-1}-w_{t-1}\|_2 \le 2D.
\end{aligned}
\label{eq:app_bound_ab}
\end{equation}
Combining \eqref{eq:app_diff_sq}--\eqref{eq:app_bound_ab} yields
\begin{equation}
\|u_t-w_{t-1}\|_2^2-\|u_{t-1}-w_{t-1}\|_2^2
\;\le\;
2D\,\|u_t-u_{t-1}\|_2.
\label{eq:app_key_diff_bound}
\end{equation}
Substitute \eqref{eq:app_key_diff_bound} into \eqref{eq:app_telescop_prep}:

\begin{align}
\sum_{t=1}^{T}\big(F_t(w_t)-F_t(u_t)\big)
&\le
\frac{1}{2\eta}\|u_1-w_0\|_2^2
+
\frac{1}{2\eta}\sum_{t=2}^{T} 2D\,\|u_t-u_{t-1}\|_2\nonumber\\
&=
\frac{1}{2\eta}\|u_1-w_0\|_2^2
+
\frac{D}{\eta}\sum_{t=2}^{T}\|u_t-u_{t-1}\|_2\nonumber\\
&=
\frac{1}{2\eta}\|u_1-w_0\|_2^2
+
\frac{D}{\eta}V_T(\{u_t\}),
\end{align}
which proves \eqref{eq:app_dyn_regret}. Taking $u_t=w_t^\star$ gives \eqref{eq:app_dyn_regret_star}.
\end{proof}

\subsection{Why Anchoring Reduces the Drift Term: A Contraction Result}
The dynamic regret bound in Theorem~\ref{thm:app_dyn_regret} depends on the comparator path length
$V_T(\{u_t\})$. We now show, in a standard quadratic (strongly convex) surrogate model, that anchoring
provably contracts the \emph{optimal} trajectory, thereby reducing the drift-controlled term when the
anchor itself changes slowly.

\paragraph{Long-term quadratic surrogate.}
Consider a long-term-only surrogate
\begin{equation}
f_t^{\mathrm{long}}(u)
\;:=\;
\frac{1}{2}\|A u - b_t\|_2^2,
\qquad
H:=A^\top A,
\qquad
\mu I \preceq H \preceq L I,
\label{eq:app_quad_model}
\end{equation}
with $\mu>0$ (strong convexity) and $L\ge \mu$ (smoothness). Define the unanchored minimizer
$u_t^\circ := \arg\min_u f_t^{\mathrm{long}}(u)$ and the anchored minimizer
\begin{equation}
u_t^\star
\;:=\;
\arg\min_{u\in\mathbb{R}^{d_\ell}}
\Big\{
f_t^{\mathrm{long}}(u)
+
\frac{\lambda_s}{2}\|u-c_t\|_2^2
\Big\}.
\label{eq:app_anchor_opt_def}
\end{equation}

\begin{proposition}[Anchor-induced contraction of the optimal path]
\label{prop:app_contraction}
Under \eqref{eq:app_quad_model}, the anchored optimum satisfies
\begin{equation}
\|u_t^\star-u_{t-1}^\star\|_2
\;\le\;
\kappa(\lambda_s)\,\|u_t^\circ-u_{t-1}^\circ\|_2
\;+\;
\alpha(\lambda_s)\,\|c_t-c_{t-1}\|_2,
\label{eq:app_contraction_main}
\end{equation}
where
\begin{equation}
\begin{aligned}
\kappa(\lambda_s) &:= \big\|(H+\lambda_s I)^{-1}H\big\|_{\mathrm{op}}
                   \le \frac{L}{\mu+\lambda_s},\\
\alpha(\lambda_s) &:= \big\|\lambda_s(H+\lambda_s I)^{-1}\big\|_{\mathrm{op}}
                   \le \frac{\lambda_s}{\mu+\lambda_s}.
\end{aligned}
\label{eq:app_kappa_alpha}
\end{equation}
Consequently, defining
$V_T^\circ := \sum_{t=2}^T\|u_t^\circ-u_{t-1}^\circ\|_2$ and
$V_T^c := \sum_{t=2}^T\|c_t-c_{t-1}\|_2$, we obtain
\begin{equation}
V_T^\star
\;:=\;
\sum_{t=2}^T\|u_t^\star-u_{t-1}^\star\|_2
\;\le\;
\kappa(\lambda_s)\,V_T^\circ
\;+\;
\alpha(\lambda_s)\,V_T^c.
\label{eq:app_path_contracted}
\end{equation}
\end{proposition}

\begin{proof}
First note the closed forms. The unanchored minimizer solves $H u = A^\top b_t$, hence
\begin{equation}
u_t^\circ = H^{-1}A^\top b_t.
\label{eq:app_u_circ}
\end{equation}
The anchored minimizer solves $(H+\lambda_s I)u = A^\top b_t + \lambda_s c_t$, hence
\begin{equation}
u_t^\star
=
(H+\lambda_s I)^{-1}\big(A^\top b_t + \lambda_s c_t\big)
=
(H+\lambda_s I)^{-1}\big(Hu_t^\circ + \lambda_s c_t\big).
\label{eq:app_u_star}
\end{equation}
Subtract \eqref{eq:app_u_star} at times $t$ and $t-1$:
\begin{align}
u_t^\star-u_{t-1}^\star
&=
(H+\lambda_s I)^{-1}\Big(H(u_t^\circ-u_{t-1}^\circ)+\lambda_s(c_t-c_{t-1})\Big)\nonumber\\
&=
(H+\lambda_s I)^{-1}H\,(u_t^\circ-u_{t-1}^\circ)
\;+\;
\lambda_s(H+\lambda_s I)^{-1}(c_t-c_{t-1}).
\label{eq:app_diff_decomp}
\end{align}
Taking norms and applying the operator norm bound gives \eqref{eq:app_contraction_main} with the
definitions in \eqref{eq:app_kappa_alpha}. It remains to bound $\kappa(\lambda_s)$ and $\alpha(\lambda_s)$.

Diagonalize $H=Q\Lambda Q^\top$ with $\Lambda=\mathrm{diag}(\lambda_i)$ and $\lambda_i\in[\mu,L]$.
Then
\[
(H+\lambda_s I)^{-1}H = Q(\Lambda+\lambda_s I)^{-1}\Lambda Q^\top,
\]
whose eigenvalues are $\frac{\lambda_i}{\lambda_i+\lambda_s}\le \frac{L}{\mu+\lambda_s}$, hence
$\kappa(\lambda_s)\le \frac{L}{\mu+\lambda_s}$.
Similarly,
\[
\lambda_s(H+\lambda_s I)^{-1} = Q\,\mathrm{diag}\!\Big(\frac{\lambda_s}{\lambda_i+\lambda_s}\Big)Q^\top,
\]
so its operator norm is at most $\max_i \frac{\lambda_s}{\lambda_i+\lambda_s}\le \frac{\lambda_s}{\mu+\lambda_s}$,
yielding $\alpha(\lambda_s)\le \frac{\lambda_s}{\mu+\lambda_s}$.
Summing \eqref{eq:app_contraction_main} over $t=2,\dots,T$ yields \eqref{eq:app_path_contracted}.
\end{proof}

\end{document}